\documentclass[sigconf]{acmart}

\AtBeginDocument{%
  \providecommand\BibTeX{{%
    \normalfont B\kern-0.5em{\scshape i\kern-0.25em b}\kern-0.8em\TeX}}}

\usepackage{enumitem}
\usepackage{times}
\usepackage{soul}
\usepackage{url}
\usepackage{hyperref}
\usepackage[utf8]{inputenc}
\usepackage{subcaption}
\usepackage{graphicx}
\usepackage{amsmath}
\usepackage{amsthm}
\usepackage{amssymb}
\usepackage{bbm}
\usepackage{booktabs}
\usepackage{mathtools}
\usepackage{algorithm}
\usepackage{algorithmic}

\usepackage{mathabx}
\newtheorem{definition}{Definition}

\newtheorem{theorem}{Theorem}
\newtheorem{lemma}{Lemma}
\newtheorem{proposition}{Proposition}

\usepackage{listings}
\usepackage{xcolor}
\definecolor{codegreen}{rgb}{0,0.6,0}
\definecolor{codegray}{rgb}{0.3,0.3,0.3}
\definecolor{codepurple}{rgb}{0.58,0,0.82}
\definecolor{backcolour}{rgb}{0.95,0.95,0.92}

\lstdefinestyle{mystyle}{
    commentstyle=\color{codegreen},
    keywordstyle=\color{magenta},
    stringstyle=\color{codegray},
    basicstyle=\ttfamily\footnotesize,
    breakatwhitespace=false,         
    breaklines=true,                 
    captionpos=b,                    
    keepspaces=true,                 
    showspaces=false,                
    showstringspaces=false,
    showtabs=false,                  
    tabsize=2
}

\begin{document}
\settopmatter{printacmref=false}
\setcopyright{none}
\renewcommand\footnotetextcopyrightpermission[1]{}

\title{Tradeoffs in Streaming Binary Classification under Limited Inspection Resources }


\author{Parisa Hassanzadeh}
\affiliation{%
  \institution{J.P. Morgan AI Research}
  \country{}
}
 \email{parisa.hassanzadeh@jpmorgan.com}

\author{Danial Dervovic}
\affiliation{%
  \institution{J.P. Morgan AI Research}
  \country{}
  }
\email{danial.dervovic@jpmorgan.com}

\author{Samuel Assefa}
\affiliation{%
  \institution{U.S. Bank AI Innovation}
  \country{}
  }
\email{samuel.assefa@usbank.com}

\author{Prashant Reddy}
\affiliation{%
  \institution{J.P. Morgan AI Research}
 \country{}
  }
  \email{prashant.reddy@jpmorgan.com}

\author{Manuela Veloso}
\affiliation{%
 \institution{J.P. Morgan AI Research}
 \country{}
 }
 \email{manuela.veloso@jpmorgan.com}


\begin{abstract}
Institutions are increasingly relying on machine learning models to identify and alert on abnormal events, such as fraud, cyber attacks and system failures. These alerts often need to be manually investigated by specialists. Given the operational cost of manual inspections, the suspicious events are selected by alerting systems with carefully designed thresholds. In this paper, we consider an imbalanced binary classification problem, where events arrive sequentially and only a limited number of suspicious events can be inspected. We model the event arrivals as a non-homogeneous Poisson process, and compare  various suspicious event selection methods including those based on static and adaptive thresholds. For each method,  
we analytically characterize the tradeoff between the minority-class detection rate and the inspection capacity as a function of the  data class imbalance and the classifier confidence score densities. We implement the selection methods on a real public fraud detection dataset and compare the empirical results with  analytical bounds. Finally, we investigate how class imbalance and the choice of classifier impact the tradeoff. 
\end{abstract}

\maketitle
\pagestyle{plain}

\section{Introduction}
Automated information processing and decision-making systems used in finance, security, quality control and medical applications use machine learning models for monitoring sequentially arriving events for malicious activities or abnormalities. Identifying such events in a timely manner can be crucial in preventing unfavorable outcomes, such as monetary loss due to fraud in retail banking or data breaches due to cyber attacks. While missing an abnormal event can have adverse consequences - since such events are sporadic and investigating events entails operational costs and can lead to processing delays - these systems are restricted in the number of risky events they select for manual inspection.  
 
Many machine learning classification algorithms predict a score (often a probability) for each data sample  representing the algorithm's confidence about its class membership. In a binary classification problem, the class labels are derived from converting the predicted scores to binary labels using a \textit{(decision) threshold}. Adjusting the threshold, especially in settings with severe class-imbalance or when the misclassification of one class outweighs the misclassification of the other class, can profoundly impact the classifier performance (e.g., True Positive (TP)-False Positive (FP) tradeoff)  \cite{provost2000machine}.  
The threshold is generally tuned using a grid search across a range of thresholds,  or it is computed from the receiver operating characteristic (ROC) curve or the precision-recall curve in highly skewed datasets \cite{he2009learning}.

Given the imbalanced nature of data in this domain, which makes learning classifiers that efficiently discriminate among the minority and majority class difficult, and the limited resources available for inspecting time-sensitive risky events, we are interested in understanding the relationship between the rate of detection from the minority class (i.e., the fraction of samples from the minority class selected for inspection) and the inspection budget. Specifically, we focus on applications that involve real-time processing and decision-making where an abnormal event can only be inspected at the time of arrival, and we investigate how different selection policies based on classifier predictions operate in terms of the limited inspection budget rather than the decision threshold.

Point processes, such as the Poisson process, have been widely used for modeling event arrivals at random times in various applications, such as arrivals in call centers \cite{kim2014call} 
system failures \cite{rausand2003system}, network traffic models \cite{chandrasekaran2009survey}, and in financial modeling \cite{giesecke2004credit,kou2002jump}. We note that Poisson processes are not suitable in settings with scheduled arrivals (e.g., doctor appointments),  intentionally separated events (e.g., plane landings), or events that arrive in groups, (e.g., at a restaurant where group members are not independent from one another). One can perform statistical tests on the data to confirm whether the arrivals can be modeled as a non-homogeneous Poisson process as described in \cite{kim2014choosing}. In the setting considered in this work, events (e.g., transactions) arrive and are processed independently from one another, and therefore, their arrival can be modeled according to a non-homogeneous Poisson process. The analysis can easily be extended to the more general renewal process \cite{daley2007introduction}.

\paragraph{Contributions}
We consider an imbalanced  binary classification problem where the goal is to select a limited number of sequentially arriving data samples that are most likely to belong to the minority class. 
We present the problem in the context of fraud detection in financial transactions, but our results apply to the general imbalanced binary classification problem. Our contributions are as follows: 


\begin{itemize}[leftmargin=*]
\item We break the problem into two tasks: learning a classifier from data, and using the classifier predictions to sample sequential arrivals for inspection.  We focus on the second task, and study the tradeoff between the minority class detection rate and the inspection budget, for a given learned classifier. This tradeoff can be used to determine how many samples need to be inspected to achieve a certain detection rate from the minority class.

\item We assume that events arrive periodically according to a NHPP, and focus on four selection (decision) strategies: sampling based on static and dynamic thresholds, random sampling and sampling in batches. For each method, we analytically characterize the minority-class detection rate-inspection capacity  tradeoff. 

\item For the case of sampling with respect to a static threshold, we determine the optimal threshold value that maximizes the minority-class detection rate for a given inspection capacity.

 \item We use a publicly available fraud detection dataset to learn a classifier and estimate the time-varying arrival rate function of the NHPP, and compare the empirical results from each sampling technique with our analytical bounds. We show that using dynamic thresholds operates very closely to the upper performance limit resulting from sampling in batches, especially for very small inspection capacities.
 
 \item For this dataset, we investigate how  class imbalance and the predictive power of the classifier,  affects the  tradeoff, and compare the empirical results against an upper bound on the end-to-end problem when considering both tasks of learning and operational decisions jointly.

\end{itemize}


\vspace{-.25cm}
\section{Related Work}
Most work on detecting samples from the minority class focus on learning optimal classifiers with respect to a given performance metric, such as the F1-score or the area under the ROC curve, and then satisfy the inspection constraints by adjusting the decision threshold. The work in \cite{koyejo2014consistent} studies the optimal fixed threshold selection problem for binary classification with respect to various performance metrics. In \cite{shen2020deep}, the authors consider a dynamic environment and model the threshold tuning as a sequential decision making problem. They use reinforcement learning to adaptively adjust the thresholds  by maximizing a reward in terms of the net monetary value of missed and detected frauds when  restricted to a fixed inspection capacity. 
The work in \cite{li2015scalable} studies  the adversarial binary classification problem with operational constraints (e.g.,  inpsection budget), where an intelligent adversary  attempts to evade the decision policies. By modeling the problem as a Stackleberg game, they determine the optimal randomized operational policy that abides by the constraints. In other related work with  dynamic environments, \cite{houssou2019adaptive} considers a fraud detection setting where rare fraudulent events arrive from a Poisson process with a \textit{parametric} arrival function estimated from the data, and the goal is to predict the arrival of a new fraudulent event.

More recently, \cite{dervovic2021non} adopted the sequential assignment algorithm of \cite{albright1974optimal} in the fraud detection setting such that the overall value of detected fraudulent transactions is maximized. In this paper, we use this algorithm  to find adaptive thresholds for transactions arriving at random times according to a Poisson process.  Given the sequential nature of arrivals in information processing and decision-making applications, this algorithm allows us to directly account for the limited inspection capacity when deciding to inspect transactions based on the output of a machine learning model.  

\section{Problem Formulation}\label{sec:problem formulation}
Consider a transaction (e.g., payment, credit card purchase) fraud detection setting, where transactions arrive sequentially at random times over a finite time horizon $[0,\tau]$ according to a Non-Homogeneous Poisson Process (NHPP) with a continuous arrival rate (intensity) function $\lambda(t)$, where $t\in[0,\tau]$ denotes the time of arrival. To each transactions $i$ we associate a  triplet $({X}_i, Y_i, t_i)$, where ${X}_i\in \mathbb{R}^d$ is a random variable representing the observed features of transaction $i$, $Y_i\in\{0,1\}$ indicates if the transaction is fraudulent, and $t_i\in[0,\tau]$ denotes its random arrival time. We assume that transactions are independent from one another and that there is significant class imbalance such that $ {P}_Y(Y=1) = \beta \ll 0.5$, i.e., there are considerably less  fraudulent transactions compared to non-fraudulent ones.    
 
There is a binary classifier $G:\mathbb{R}^d\rightarrow\mathbb{R}$, that assigns a random value $S=G({X})\in[0,1]$ to a transaction with feature vector ${X}$, which represents the classifier's confidence that a transaction is fraudulent. We have an inspector (decision-maker) that can investigate a transaction, and determine whether it is fraudulent without error; however, the inspector  is only able to investigate a limited number of transactions during  $[0,\tau]$. Therefore, the inspector needs to decide which transactions should be selected for inspection in order to detect as many fraudulent transactions as possible given its limited inspection resources. Note that since transactions $({X}_i, Y_i, t_i)$ and $({X}_j, Y_j, t_j)$ are independent, then for any classifier $G$, the corresponding  scores $S_i = G({X}_i)$ and $S_j = G({X}_j)$ are also independent.

\paragraph{Transaction Arrival Process:} 
We assume transactions arrive according to a NHPP with  rate  function  $\lambda(t)$, and   cumulative rate function  $\Lambda(t) \coloneqq \int_{0}^t \lambda(u) du$. The number of transactions in interval $[0,t]$ is a random variable $N(t)$ with Poisson distribution parametrized by $\Lambda(t)$, and the expected number of arrivals in  $[0,\tau]$ is $\Lambda(\tau)$\footnote{A NHPP is denoted by $N(t_1,t_2)$,  corresponding to the number of arrivals in $(t_1,t_2)$. We use $N(t)$ to denote the arrivals in  $(0,t)$.}. The rate function $\lambda(t)$ and cumulative rate function $\Lambda(t)$ can be estimated from several $i.i.d.$ observed realizations of the  process $N(t)$ over $[0, \tau]$, using non-parametric estimators as proposed in  
\cite{lewis1976statistical,arkin2000nonparametric}, or through parametric  methods as in \cite{lee1991modeling,kuhl2000least}. In our experiments in Sec.~\ref{sec:experiments}, we use the heuristic estimator proposed in \cite{lewis1976statistical}. 

\vspace{-.3cm}
 \subsection{Objective}\label{subsec:objective}
 The inspector has limited resources, and can only select and investigate a fraction $k\in[0,1]$ of the incoming transactions in $[0,\tau]$, which we refer to as the \textit{inspection capacity}. We assume that for a given  capacity $k$, it selects a fraction $k$ of the expected number of arrivals equal to  $n_k\coloneqq \lfloor k \Lambda(\tau)\rfloor$ transactions. Our goal is to evaluate how well a given sampling method is able to choose fraudulent transactions based on the scores of a binary classifier $G(.)$, when there is limited inspection capacity.  Specifically, we define the \textit{fraud detection rate} as a function of $k$, denoted by $\Psi(k)$, to be the expected fraction of true frauds selected for inspection, given as  
 \begin{align}
    \Psi(k) &=  \mathbb{E}_{Y,N}    \bigg[  \frac{\mathbb{E}_{X}\Big[ \sum_{i\in \mathcal{I}(k)} Y_i \Big] }{ \sum_{i=1}^{N(\tau)}  Y_i  } \bigg] 
    \label{eq:objective}
 \end{align}
 where ${\mathcal I}(k)$ is the set of sequentially arriving transaction indices selected for inspection using classifier $G(.)$, which depends on the arrival process $N(t)$. 
 Note that $\Psi(k)$ in \eqref{eq:objective} is closely related to the true positive rate (TPR) 
 of the classifier, with the slight difference that it is defined for the setting with streaming data constrained by operational resources and is defined with respect to the capacity $k$.

 \subsection{Inspection Sampling Methods}\label{subsec:sampling methods}
We briefly describe the various decision-making methods considered that are used to select transactions for inspection.
\paragraph{{Static Thresholds:}} The inspector determines a fixed threshold $\alpha$, and will only inspect an arriving transaction if its score satisfies $G({X})\geq \alpha$, and if it has not exhausted its inspection capacity. If a transaction is not selected for inspection at the time of arrival, it will not be inspected at a later time. The inspector determines the threshold while accounting for the arrival of transactions and with respect to the classifier performance.  Note that if the threshold $\alpha$ is set too high, then the inspector may not select enough transactions for inspection, and if it is set too low, the inspector may select more non-fraudulent transactions initially, using up its inspection capacity too early, and therefore, will not be able to inspect transactions with high scores arriving later.

\paragraph{{Dynamic Thresholds:}} In this case, the inspector determines a time-dependent threshold $\alpha(t)$, and inspects a transaction arriving at time $t$ if its associated score satisfies $G({X})\geq\alpha(t)$. Similar to the static threshold, this time-varying threshold is determined such that the inspector selects the transactions that are most likely to be fraudulent given the classifier performance and the arrival process.  
 
\paragraph{{Random Sampling:}} The inspector disregards the confidence score assigned by the classifier, and selects transactions uniformly at random. This method is equivalent to a worst-case scenario where a no-skill classifier assigns a score of $G({X})=0.5$ to all  
${X}\in\mathbb{R}^d$.

\paragraph{{Batch Processing:}} Assume that the inspector can process and investigate transactions in batches, and there is no need to select transactions instantaneously at the time of arrival. Then, the inspector will select the set of  transactions with the highest scores $G({X})$ at time $\tau$. Batch processing is not a practical method for the setup considered here,  as there is a strict requirement for timely decision-making. This method provides an upper
performance limit for any realistic method using a classifier for real-time decision-making, and is therefore included in our analysis.

The following section describes each sampling method in detail and presents its corresponding fraud detection rate.

\section{Fraud Detection Rate-Inspection Capacity Tradeoff}\label{sec:prob}
In this section, we compute the expected fraction of frauds that are selected, and therefore detected, with each of the methods described in Sec.~\ref{subsec:sampling methods}. We denote the probability density function ({PDF}) of the classifier score $S=G({X})$ assigned to a transaction with feature vector ${X}$ and label $Y$ by $f_{S}(s)  \coloneqq f_{X,Y}(G({x}),y)$. Let $f_{0}(s) \coloneqq  f_{X|Y}(G({x})|Y=0)$ and ${f_{1}}(s) \coloneqq  f_{X|Y}(G({x})|Y=1)$ denote the PDF of the score assigned 
to a non-fraudulent and fraudulent transaction, respectively\footnote{We do not explicitly show  random variable $S$ as a subscript of $f_0(s)$ and $f_1(s)$ hereafter.}. Accordingly, $F_{0}(s)$ and ${F_{1}}(s)$ denote the cumulative distribution function (CDF) of classifier scores assigned to non-fraudulent and fraudulent transactions. 

Based on Proposition~\ref{prop:splitting} given in Appendix~\ref{app:nhhp}, $N(t)$ with rate $\lambda(t)$ can be split into two \textit{independent} sub-processes as follows:

\begin{itemize}
\item Process $N_0(t)$ with rate  $\lambda_0(t)=(1-\beta)\lambda(t)$ represents arrival of non-fraudulent transactions, and the random number of arrivals in $[0,\tau]$ is $N_0(\tau)$. We denote the $i.i.d.$ scores corresponding to  transactions  from $N_0(t)$ (ordered in time), by ${S}_i^{(0)}\sim f_{0}(s),\, i  = 1,\dots,{N_0(\tau)}$ with arrival times $t^{(0)}_{i}\leq t^{(0)}_{i+1}$.
    
\item Process $N_1(t)$ with rate  $\lambda_1(t)=\beta\lambda(t)$ represents arrival of fraudulent transactions, and the random number of arrivals in $[0,\tau]$ is $N_1(\tau)$. We denote the $i.i.d.$ scores corresponding to  transactions arriving from $N_1(t)$ (ordered in time), by ${S}_i^{(1)}\sim f_{1}(s),\, i  = 1,\dots,{N_1(\tau)}$ with arrival times $t^{(1)}_{i}\leq t^{(1)}_{i+1}$.
\end{itemize}

\subsection{Static Thresholds}
The inspector selects $n_k$ transactions 
as they arrive if the classifier score exceeds a predetermined threshold $\alpha$, and if it does not violate the capacity constraint. 
\begin{theorem}\label{thm:FT}
The fraud detection rate with respect to a static threshold $\alpha$, denoted by $\Psi_{FT}(k) $, is
\begin{align}
\Psi_{FT}(k)  
\geq (1-F_1(\alpha))  \min\Big\{ \frac{k-1/\Lambda(\tau)}{ 1-F_S(\alpha)} , \,1    \Big\}.
\end{align}
 
 \end{theorem}

\begin{proof}
The proof is given in Appendix~\ref{app:FT}.
\end{proof}
Theorem~\ref{thm:FT-alpha} provides the  threshold that maximizes $\Psi_{FT}(k)$.
\begin{theorem}\label{thm:FT-alpha}
Given an inspection capacity $k$, the optimal static threshold value  that maximizes the detection rate $\Psi_{FT}(k)$ is $\alpha^* = {F}_S^{-1}(1 - k)$, for which
\begin{align}
\Psi_{FT}(k)  
\geq  1-F_1({F}_S^{-1}(1 - k) ) .
\end{align}
\end{theorem} 
\begin{proof}
The proof is given in Appendix~\ref{app:FT-alpha}.
\end{proof}
 Note that $\alpha^*$ is the threshold that satisfies  the inspector capacity such that there will only be  $n_k$ transactions (on average) with scores exceeding $\alpha^*$ in $[0,\tau]$, and it is independent from the rate function $\lambda(t)$. For non-streaming settings without inspection restrictions, scores such as Youden's J statistic,  
 or the Brier score 
 are used to determine the optimal threshold.

\subsection{Dynamic Thresholds}\label{subsec:DT}
In the case of adaptive thresholds, transactions are sequentially selected for inspection  according to a time-dependent threshold $\alpha(t)$, computed with respect to the arrival process. In this work, we adopt the strategy proposed in \cite{albright1974optimal}, originally designed for assigning jobs, each with an associated random value and arriving at random times, to a limited number of operatives with non-identical productivity. An optimal sequential assignment algorithm is proposed in \cite{albright1974optimal} that maximizes the total expected reward, defined as the expected operative productivity. 
This algorithm was recently applied to a fraud detection problem in \cite{dervovic2021non}, where each transaction is a job arriving according to a non-homogeneous Poisson process, all operatives have identical productivity,  and the value of a job is defined as a function of the  transaction monetary amount and the classifier confidence score. Here, we define the job value as the classifier confidence score, but the results can be easily extended to the setting in\cite{dervovic2021non}.

The algorithm operates as follows: Let $\alpha_j(t)\geq 0$ denote a time-dependent threshold, referred to as a \textit{critical curve}, when the inspector can select $j$ transactions. If a  transaction with  score $S$ arrives at  $t$, and the inspector has $j$ inspections left in its budget, it selects the transaction if and only if $S>\alpha_j(t)$. The optimal critical curves $\{\alpha_j(t)\}$ are derived from a set of differential equations given in Theorem~\ref{thm:albright}.

\begin{theorem}\cite[Theorem 2]{albright1974optimal}\label{thm:albright}
For a total number of $n$ inspections, the optimal critical curves that select the transactions with the highest expected sum of scores,  $\alpha_1(t)\geq   \dots\geq\alpha_n(t) $, satisfy the following system of differential equations 
\begin{align}
&\frac{d\alpha_{j}(t)}{dt}  = - \lambda(t) \Big( \phi(\alpha_{j}(t)) - \phi(\alpha_{j-1}(t)) \Big), \;  j = 1,\dots, n \notag\\
& \phi(\alpha_0(t) )= 0 ,\;\alpha_j(T) = 0, \;t\in[0,T] \notag
\end{align}
where $\phi(\alpha)  \coloneqq  \int_{\alpha}^\infty (\beta-\alpha) f(\beta) d\beta$.
\end{theorem}

In our setting, the adaptive thresholds  maximize the   expected  sum  of  scores, which in turn selects the events with the highest scores, i.e., the most suspicious ones.

\begin{theorem}\label{thm:DT}
The fraud detection rate when selecting transactions as described in Sec.~\ref{subsec:DT}, denoted by $\Psi_{DT}(k) $, is
\begin{align}
\Psi_{DT}(k)\geq \frac{1}{ \Lambda({\tau})}\sum\limits_{j=1}^{n_k}\mathbb{E}_{T_1}\dots\mathbb{E}_{T_j} \Big[q_{\alpha_{\rho(j)}}\Big(\sum_{i=1}^j T_i\Big) \Big],    
\end{align}
with $\rho(j)\coloneqq n_k-j+1$, and 
\begin{align}
q_{\alpha_j}(t) \coloneqq   \, \frac{1-F_1(\alpha_j(t))}{ 1-F_S(\alpha_j(t))}, \label{eq:DT q}
\end{align}
and where thresholds $\{\alpha_j(t)\}_{j=1}^{n_k}$ are derived from Theorem~\ref{thm:albright}. For $i=1,\dots,n_k$, we have
\begin{align}
f_{T_i}(t) =  m_i\Big(\sum_{\ell=1}^{i-1} T_{\ell}+t\Big) \exp\Big(-\int_{0}^{t} m_i\Big(\sum_{\ell=1}^{i-1} T_{\ell}+u\Big)du \Big), \notag
\end{align}
with $T_0=0$ and $m_i(t) \coloneqq (1-F_S(\alpha_{\rho(i)}(t))) \lambda(t)$.
\end{theorem}
\begin{proof}
The proof is given in Appendix~\ref{app:DT}.
\end{proof}

\subsection{Random Sampling}
The inspector selects a random subset of transactions, irrespective of the classifier  scores, and as stated in the following theorem, the detection rate is a linear function of the inspection capacity.   
\begin{theorem}\label{thm:RS}
The fraud detection rate using random sampling given an inspection capacity of $k$, denoted by $\Psi_{RS}(k) $,  is
 \begin{align}
\Psi_{RS}(k)  \geq  k- 1/\Lambda(\tau).
\end{align}
\end{theorem}
\begin{proof}
The proof is given in  Appendix~\ref{app:RS}.
\end{proof}

\subsection{Batch Processing}
In this case, the inspector selects $n_k$ transactions that have the highest scores  among all $N(\tau)$ transactions. Therefore, a fraudulent transaction is selected, irrespective of its arrival time, if it is among the $n_k$ transactions with the largest scores.

\begin{theorem}\label{thm:BP}
The fraud detection rate with batch processing, denoted by $\Psi_{BP}(k) $, is
\begin{align}
&\Psi_{BP}(k) \geq \frac{1}{\beta \Lambda(\tau)}\mathbb{E}_{ N_0,N_1}  \Big[   \sum\limits_{j=1}^{\min\{n_k,N_1(\tau)\}} F_{W_{j}}  (0)\Big],   \\ 
\text{where  }&\;\;F_{ W_{j}}(w)
   = \int_{0}^1 \int_{0}^{w+s_1} f_{S^{(0)}_{\rho_0(j)}}(s_0) f_{S^{(1)}_{\rho_1(j)}}( s_1) ds_0 ds_1 ,\notag
\end{align}
with $\rho_0(j)\coloneqq N_0(\tau)-n_k+j$ and $\rho_1(j)\coloneqq N_1(\tau)-j+1$, and the PDF $f_{S^{(i)}_{\rho_i(j)}}(s)$, $i=0,1$, is derived by Lemma~\ref{lemma:orderstat} in Appendix~\ref{app:BP}.
\end{theorem}

\begin{proof}
The proof is given in Appendix~\ref{app:BP}.
\end{proof}

For general functions of $F_0(s), F_1(s)$ and $F_S(s)$, the detection rate $\Psi(k)$ may not exist in closed-form in Theorems~\ref{thm:DT} and~\ref{thm:BP}. We approximate this function through Monte Carlo experiments in Sec.~\ref{subsec:numerical}.

 \begin{figure*}%
    \captionsetup[subfigure]{aboveskip=-1pt,belowskip=-1.5em}
    \centering
    \subfloat[\centering]{{\includegraphics[width=0.45\linewidth]{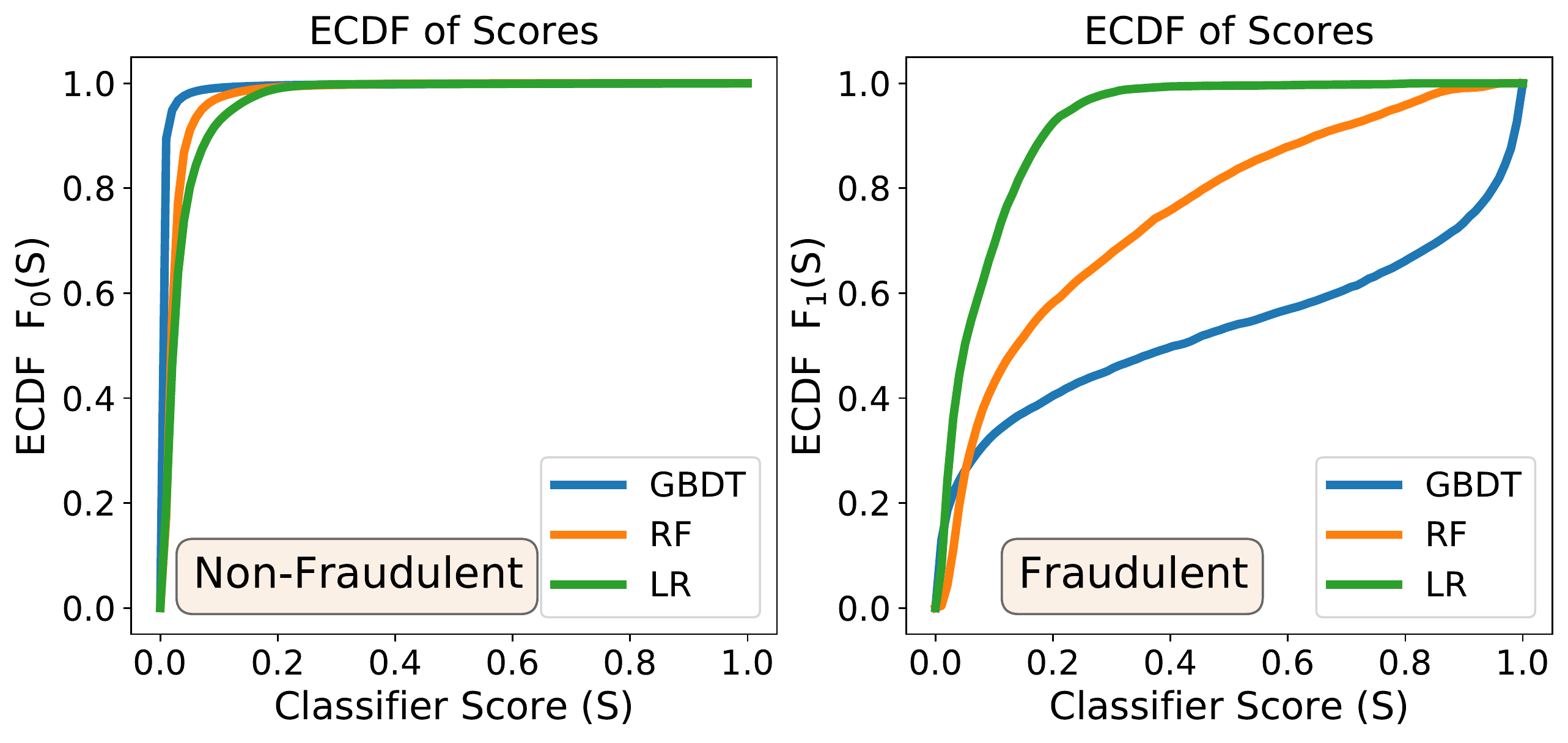} }}%
    \;
    \subfloat[\centering]{{\includegraphics[width=0.45\linewidth]{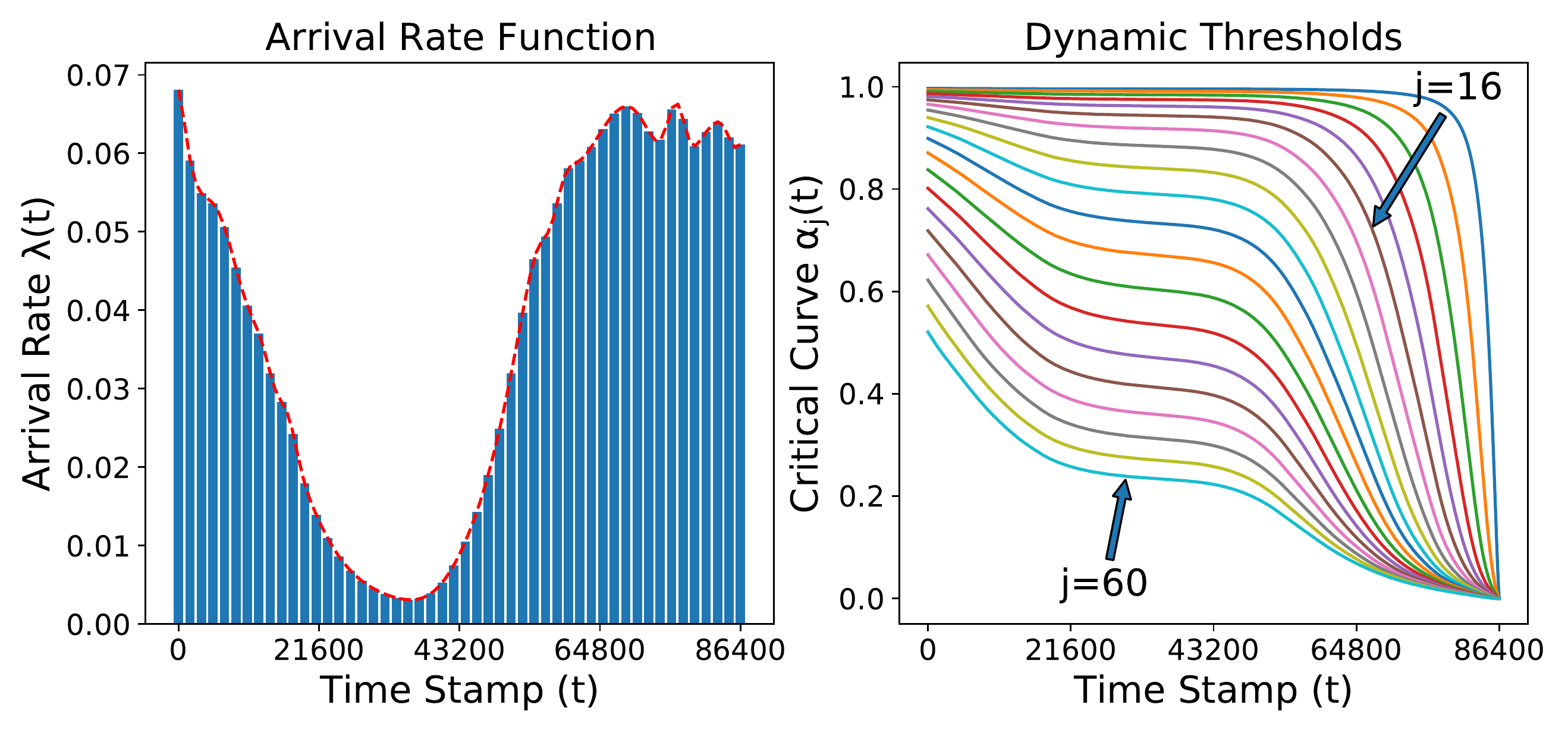} }}%
    \caption{IEEE-fraud-$3.5\%$ dataset: (a) ECDF of classifier  scores assigned to non-fraudulent (left) and fraudulent (right) transactions. (b) Estimated time-varying arrival rate $\lambda(t)$ (left), dynamic thresholds $\{\alpha_j(t)\}$ for $60$ inspections derived from Theorem~\ref{thm:albright} (right). 
    }%
    \label{fig:classifier-arrival}%
\end{figure*}

\section{Experiments}\label{sec:experiments}
In this section, we use a public dataset to compute the detection rate-inspection capacity tradeoff for each sampling method. We compare the analytical bounds on the tradeoff, derived in Sec.~\ref{sec:prob}, with the observed average tradeoff obtained empirically using each sampling method. We use the IEEE-CIS Fraud Detection (IEEE-fraud) Dataset
\footnote{Available at: \url{https://www.kaggle.com/c/ieee-fraud-detection/data}.}, provided by Vesta Corporation containing over 1 million
real-world e-commerce transactions, comprising of more than 400 feature variables, time stamps (secs) and fraud labels. The  dataset contains 183 days of transactions, with  $3.5\%$ of samples labeled fraudulent. In order to demonstrate how the class imbalance impacts the detection rate, we modify the imbalance by up-sampling (SMOTE \cite{chawla2002smote}) and down-sampling (uniformly at random) the minority class to make up $8\%$ and $1.5\%$ of the transactions, respectively. We refer to the dataset with $x\%$ frauds as IEEE-fraud-$x\%$.

\paragraph{Arrival Rate, Classifier Score Densities and  Dynamic Thresholds: }
In our experiments, we consider each interval $[0,\tau]$ to be one day with $\tau=86400$ seconds, and use a random $50\%$-$25\%$-$25\%$ split of the days as follows: $i)$ We use the first half of data to train three classifiers with different predictive powers: gradient boosted decision trees (GBDT),  random forests (RF) and  logistic regression (LR).  The training results on all three datasets using the AUC of the ROC as the evaluation metric are reported in Table~\ref{tab:train}. $ii)$ We use the second part of  data to estimate the Empirical Cumulative Distribution Function (ECDF) of the classifier scores, $F_0(s)$ and $F_1(s)$, shown in Fig.~\ref{fig:classifier-arrival}(a) for the IEEE-fraud-$3.5\%$ dataset. As expected, a more powerful classifier assigns higher scores to fraudulent transactions with much higher probability. We use the method in \cite{lewis1976statistical} to estimate the rate function $\lambda(t)$, shown in Fig.~\ref{fig:classifier-arrival}(b), which is used to compute the time-dependent thresholds used in the Dynamic Thresholds method. $iii)$ Finally, we use the last part of data for our empirical sampling experiments discussed in the following. A more detailed description of the experiment setup is provided in Appendix~\ref{app:experiments}.  
Additionally, in order to investigate how estimation errors or model assumptions, e.g., independence of classifier scores and time of arrivals, affect the empirical results, we simulate data based on our estimated $\lambda(t)$, $F_0(s)$ and $F_1(s)$, which we refer to as \textit{simulated data}.

 \begin{table}\setlength{\tabcolsep}{3pt}
\captionsetup[table]{skip=-1em}
\caption{Classifier learned on IEEE-fraud dataset.}
\label{tab:train}
\centering
\resizebox{0.85\linewidth}{!}{%
\begin{tabular}{lcccc}
\toprule
Dataset &  Parameters  & Classifier  &  AUC   & AUC   \\
  &   &    &  (Valid) &  (Test) \\
\midrule
                 &      &   GBDT &  $0.985$ & $0.979$  \\
IEEE-fraud-$8\%$ & $\beta = 0.08$ & RF     &  $0.972$  & $0.953$  \\
                 & $\Lambda(\tau) = 3378 $   & LR   & $0.809$  &$0.815$  \\
\midrule
  &   &  GBDT &   $0.966$ & $0.951$  \\
 IEEE-fraud-$3.5\%$  &  $\beta=0.035$  & RF &$0.963$ &$0.918$  \\
    &  $\Lambda(\tau) = 3219 $ &   LR   &$0.687$ &$0.697$  \\
\midrule
  &  & GBDT &   $0.945$ & $0.935$ \\
  IEEE-fraud-$1.5\%$ & $\beta = 0.015$ & RF& $0.976$  &$0.904$ \\
 
  &  $\Lambda(\tau) = 3154  $ & LR  &  $0.688$  &$0.694$ \\
\bottomrule
\end{tabular}
}

\end{table}

 \begin{figure*}%
    \captionsetup[subfigure]{aboveskip=-1pt,belowskip=-2em}
    \centering
    \subfloat[\centering ]{{\includegraphics[width=0.23\linewidth]{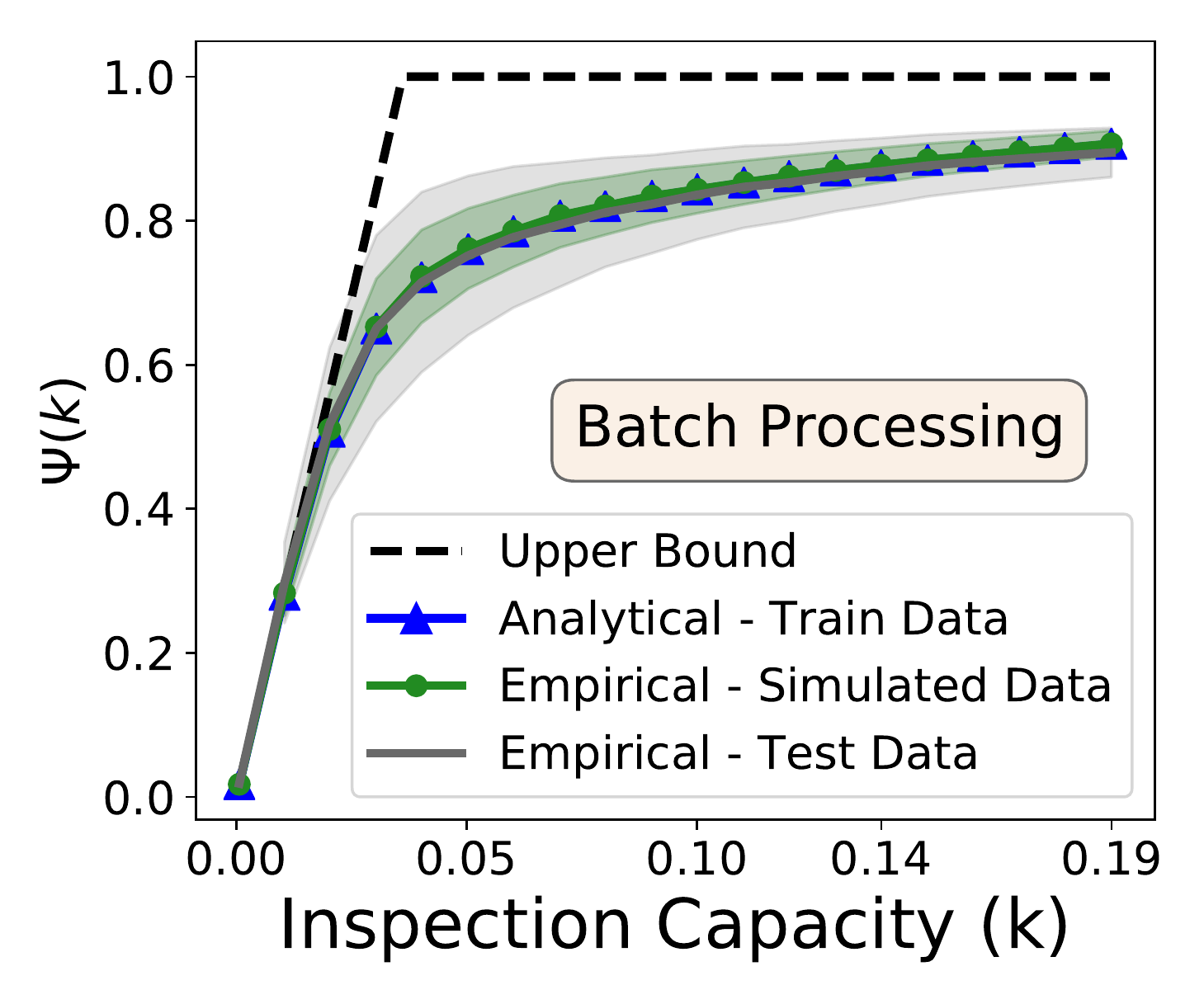} }}%
     \hspace{-0.9em}
    \subfloat[\centering ]{{\includegraphics[width=0.23\linewidth]{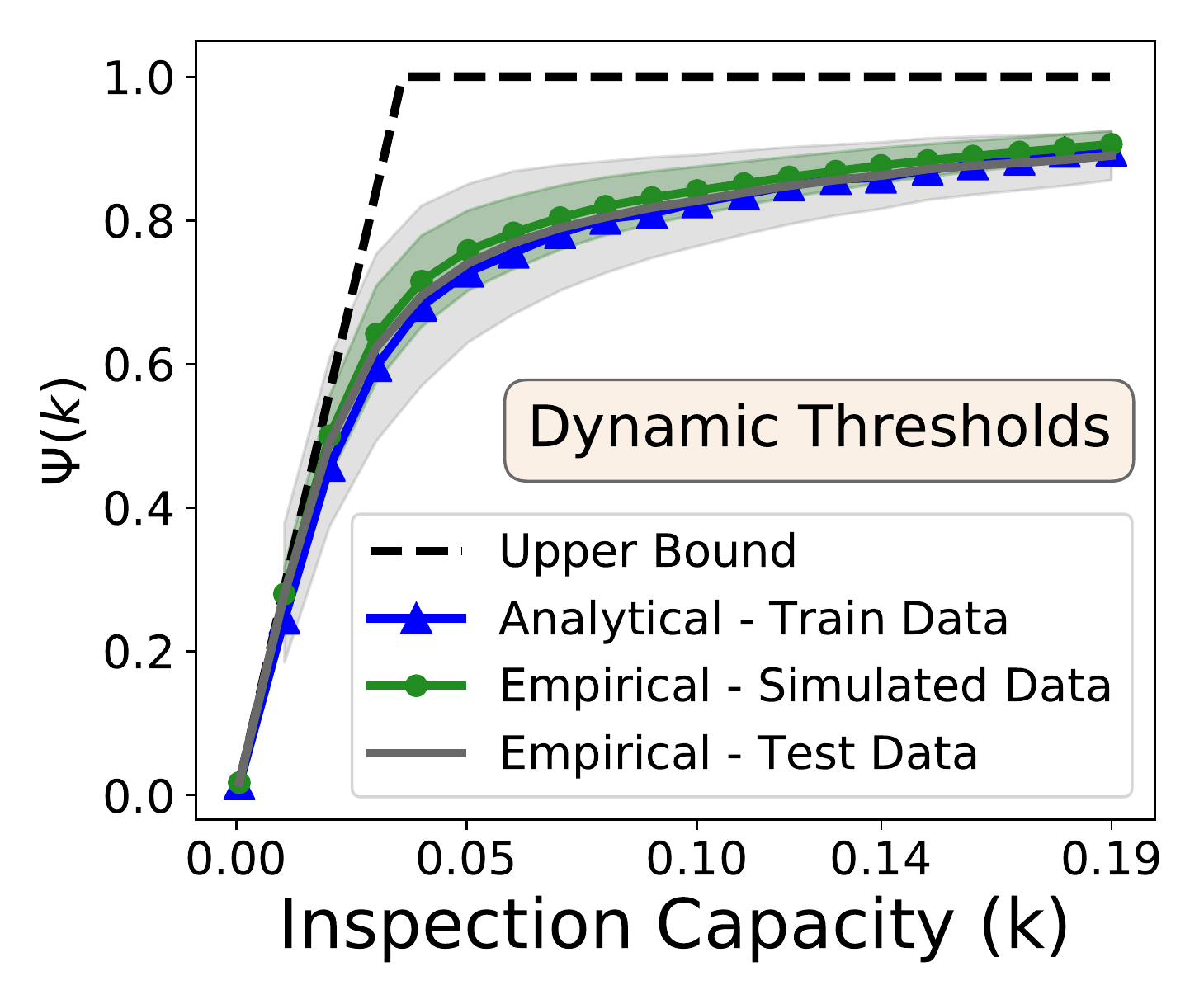} }}%
     \hspace{-0.9em}
    \subfloat[\centering ]{{\includegraphics[width=0.23\linewidth]{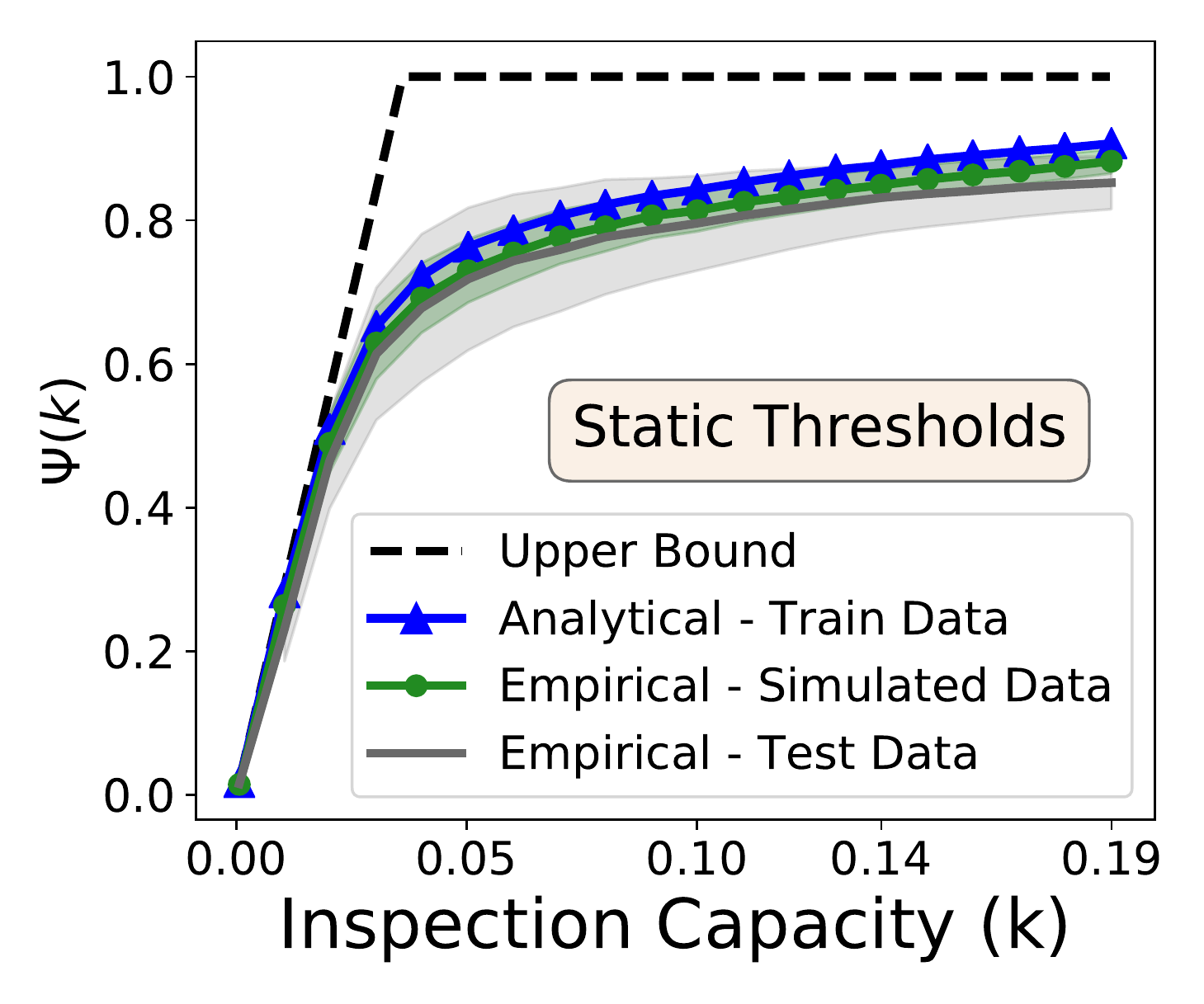} }}%
     \hspace{-0.9em}
    \subfloat[\centering ]{{\includegraphics[width=0.23\linewidth]{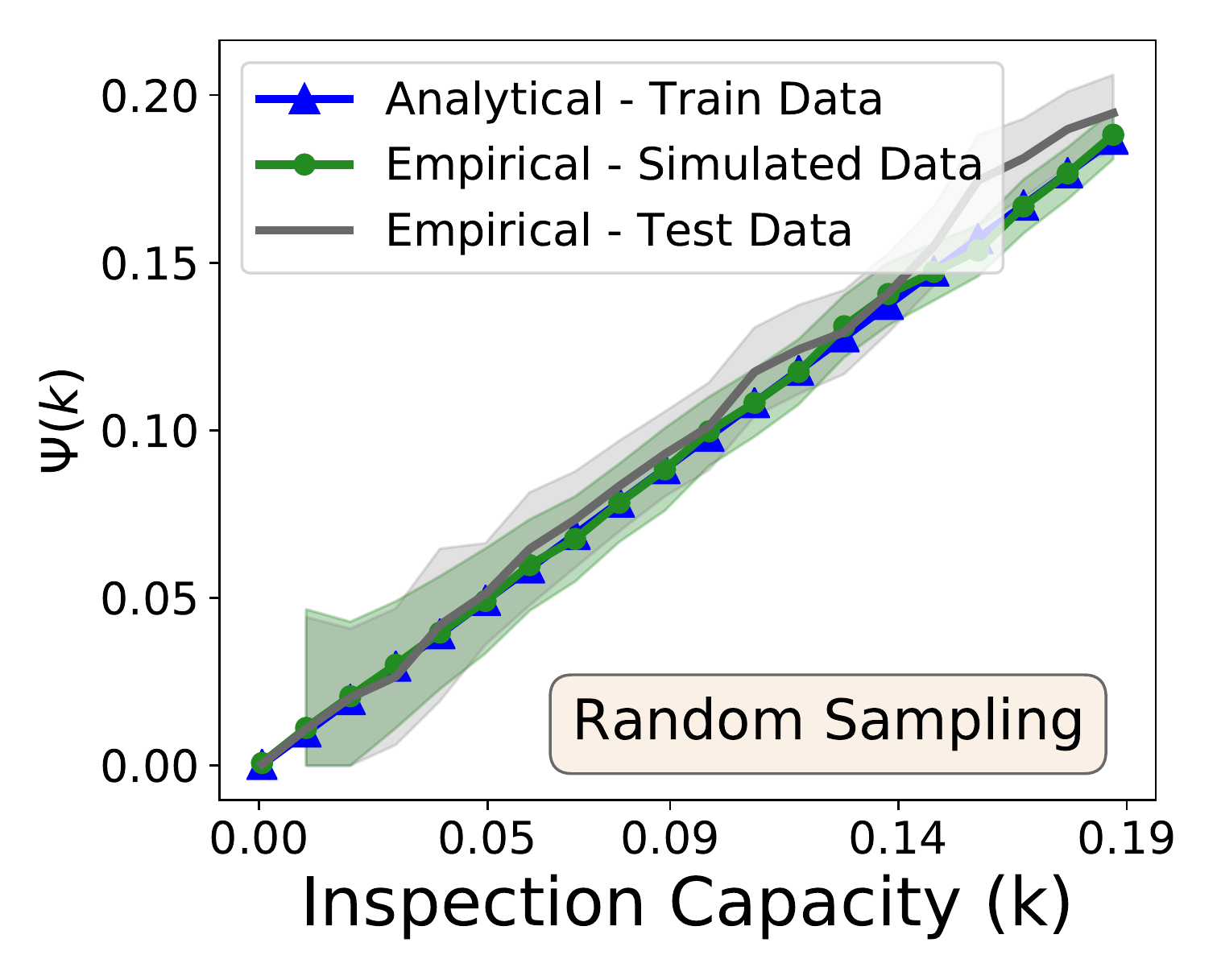} }}%
    \caption{Fraud detection rate-inspection capacity tradeoff for different sampling methods: (a) Batch Processing, (b) Dynamic Thresholds, (c) Static Thresholds, and (d) Random Sampling. The empirical results are within a small margin from the analytical bounds. 
    }%
    \label{fig:analytical}%
\end{figure*}

\subsection{Analytical Results}\label{subsec:numerical}
Fig.~\ref{fig:analytical} (a)-(d) displays the results on the IEEE-fraud-$3.5\%$ dataset with the GBDT classifier for $i)$ batch processing (BP), 
$ii)$ dynamic thresholds (DT), $iii)$ static thresholds (ST), and $iv)$ random sampling (RS), which corresponds to a no-skill classifier, 
for $k \in [0,\,0.2]$ equivalent to inspecting $0$-$20\,\%$ of the expected arrivals each day. The dashed line delineates an (not necessarily achievable) upper bound on the entire tradeoff when the learning of the classifier is also taken into account, derived in Appendix~\ref{app:UB}.  For each method, the expected tradeoff derived analytically, is very close to the experimental results on the test data and the simulated data. Specifically, the curves match almost perfectly for batch processing given that it is independent of the arrival process. With dynamic thresholds the difference between the analytical and empirical curves is much smaller compared to the static thresholds since the mismatch between the estimated arrival rate $\lambda(t)$ with the actual arrivals of fraudulent transactions affects the analytical bounds more in the case of fixed thresholds. Finally, as stated in Theorem~\ref{thm:RS}, for random sampling, the detection rate equals the inspection capacity $k$. The experiments on a real dataset show that the NHPP formulation of arrivals could be used for practical applications when inspecting streaming data.

\begin{figure}%
    \centering
\includegraphics[width=0.57\linewidth]{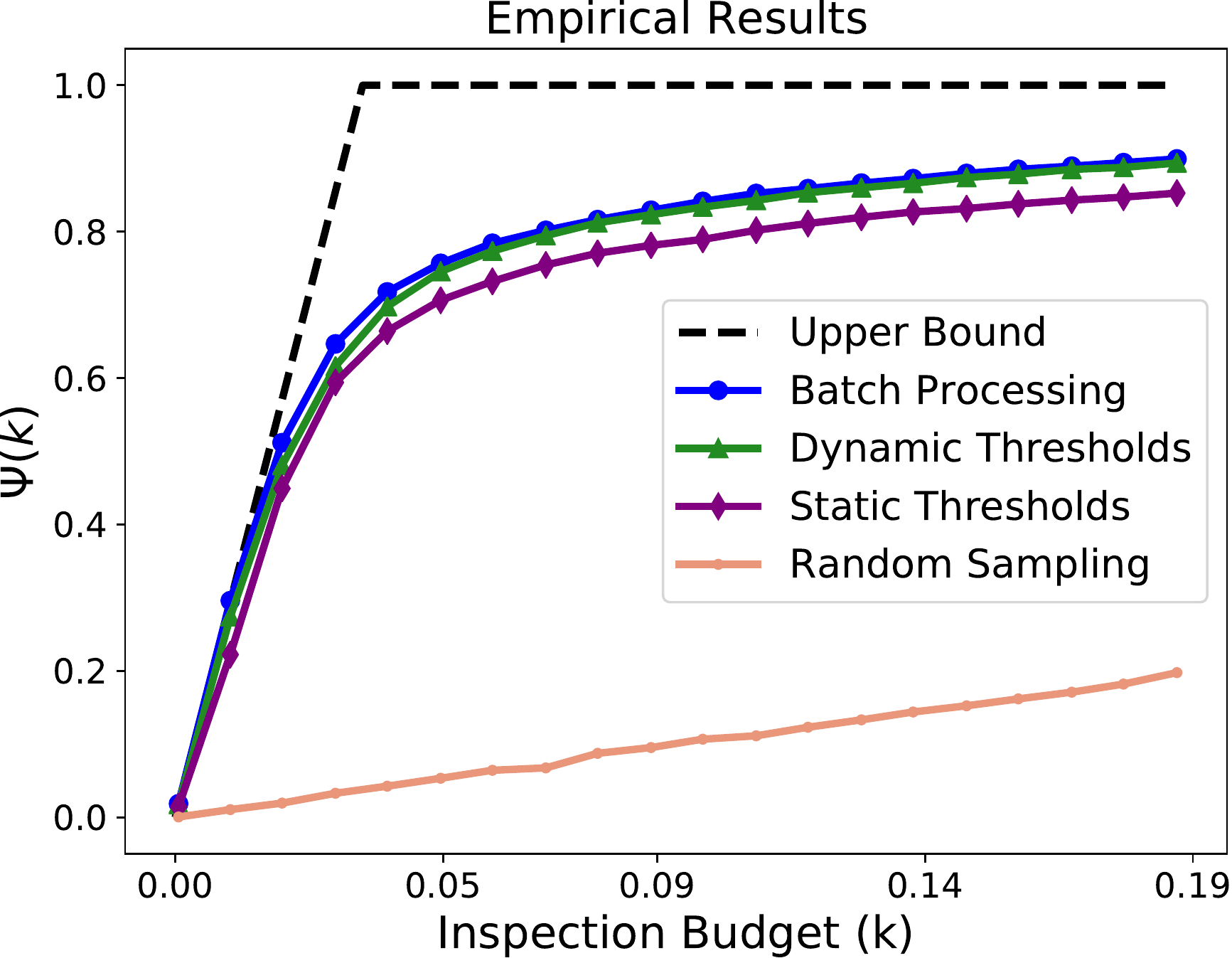}%
    \caption{IEEE-fraud-$3.5\%$ dataset. For the learned GBDT classifier, the dynamic thresholds method operates very closely to batch processing, and with the optimal fixed threshold defined in Theorem \ref{thm:FT-alpha}, the static thresholds method is within a small margin. Both methods approach the upper bound for very small inspection capacities.
    }%
    \vspace{-1em}
    \label{fig:methods}%
\end{figure}

 \subsection{Empirical Results}\label{subsec:realdata}
 Fig.~\ref{fig:methods} compares the tradeoff for different sampling methods on the IEEE-fraud-$3.5\%$ dataset with the GBDT classifier. As expected, given the sequential nature of arrivals, using adaptive thresholds to sample suspicious transactions outperforms using a static threshold. In fact, the dynamic thresholds method operates very closely (within a margin of $0.01$) to batch processing, which is not a practical method for timely decision-making itself and serves as an upper performance limit for all threshold-based strategies that sample transactions in real-time using a classifier. With the \textit{optimal} threshold derived in Theorem 3, the static threshold curve is within a margin of $0.05$ from the batch processing curve, and performs competitively to the dynamic thresholds method  especially for small $k$.  

 We investigate how other aspects of the imbalanced binary classification problem with limited resources, such as the minority-majority imbalance and the initial phase of learning a predictive classifier impacts the tradeoff. While we illustrate the results for dynamic thresholds, the tradeoffs for other methods show similar trends, and are provided in Appendix~\ref{app:experiments}.  

\paragraph{{Class Imbalance:}}  
Fig.~\ref{fig:effects} (a), shows the tradeoff on three datasets when using the GBDT classifier for predicting scores. Each curve is also compared with an upper bound (shown in dashed lines) when the end-to-end system is considered (see Appendix~\ref{app:UB}). As observed, with  adaptive thresholds, 
the  tradeoff matches the upper bound for very small inspection capacities $k$. As $k$ increases, the the detection rate diverges more from the upper bound in more imbalanced datasets, since it is much harder to learn powerful classifiers to discriminate between the minority and majority class samples.

\paragraph{Learning Phase:}  
Figs.~\ref{fig:effects} (b) shows the tradeoff on the IEEE-fraud-$3.5\%$ dataset for three classifiers with different predictive power. As expected, for the classifier with a higher AUC, which is able to better distinguish a fraudulent sample from a non-fraudulent one,  the tradeoff curve is strictly superior across all inspection capacities. This is especially pronounced for very small $k$, evident from the steep slope of the tradeoff corresponding to GBDT, which is tangent to the upper bound. A classifier with a larger AUC, with higher probability, will predict  a larger score for a sample from the minority class compared to the majority class. Therefore, the most suspicious samples can be prioritized for selection across the entire interval without using up the capacity too early.

\begin{figure}%
\captionsetup[subfigure]{aboveskip=-1pt,belowskip=-2em}
    \centering
    \subfloat[\centering ]{{\includegraphics[width=0.48\linewidth]{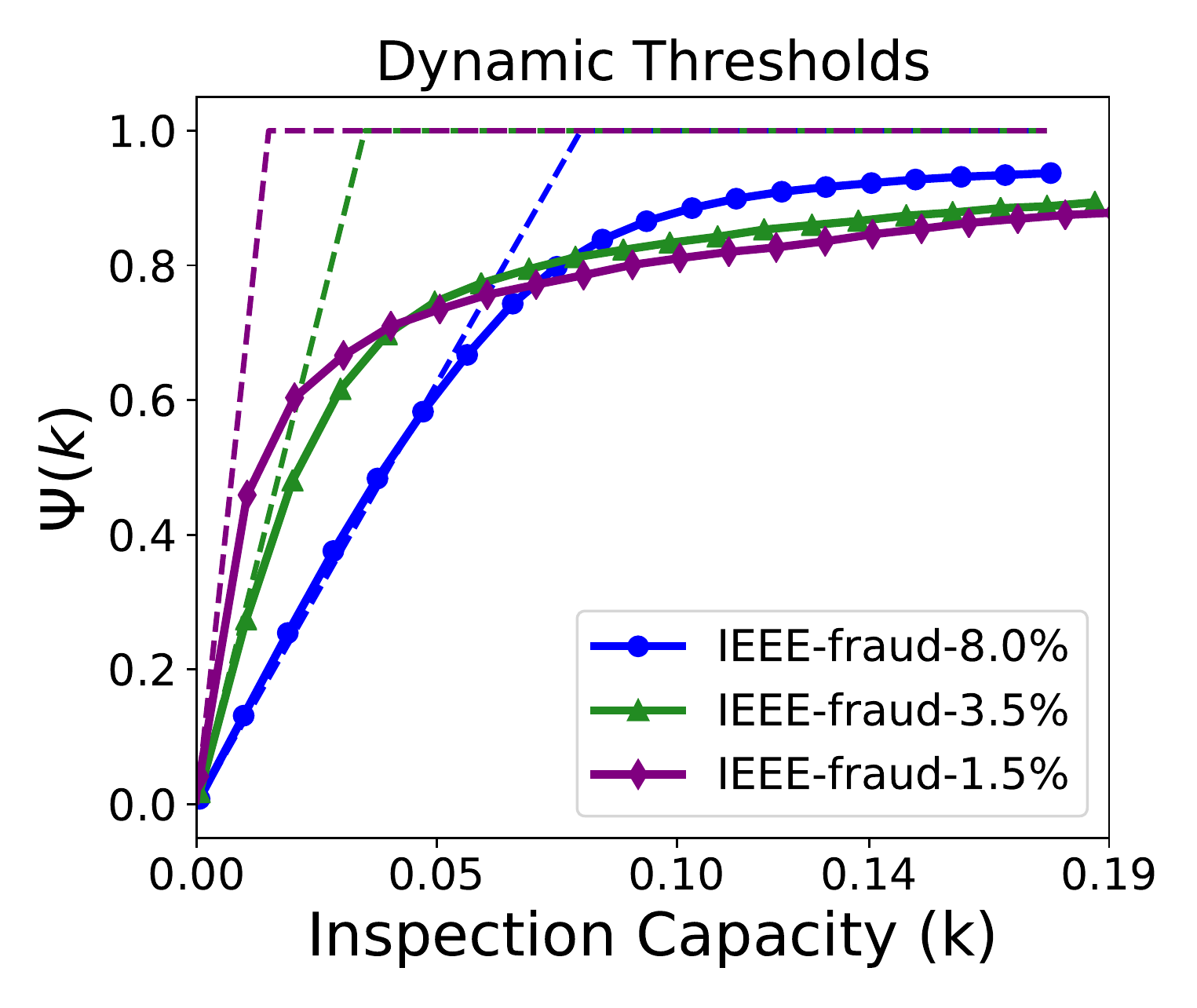} }}%
    \hspace{-0.9em}
    \subfloat[\centering ]{{\includegraphics[width=0.48\linewidth]{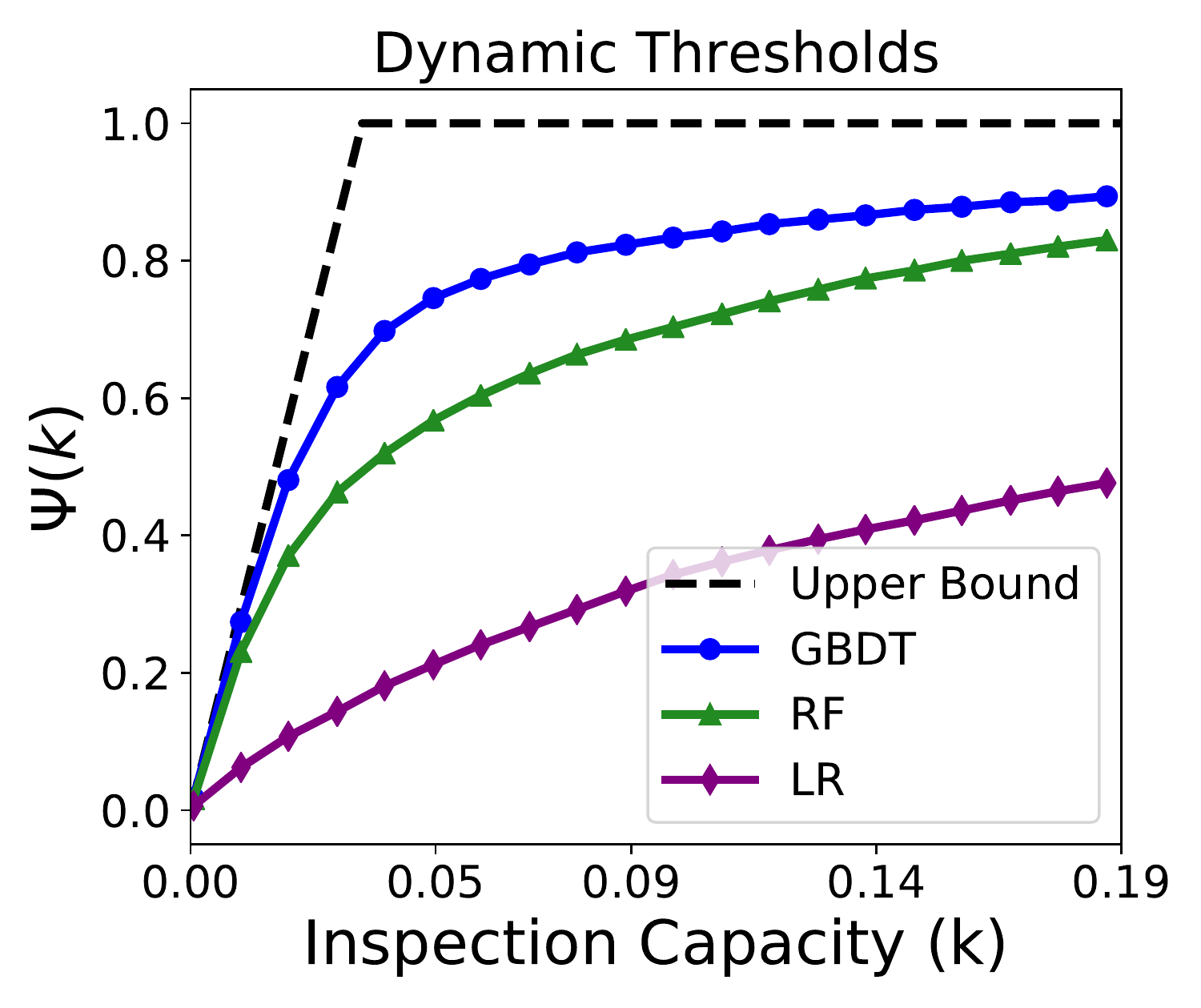} }}%
    \caption{Tradeoff with dynamic thresholds. Impact of class imbalance (a): For very small capacities, the results are very close to the upper bound, and as the capacity increases, class imbalance impacts the detection rate more severely.  Learning phase (b): a classifier with inferior predictive power (low AUC) selects non-fraudulent transactions early-on, and operates further from the upper bound.
    }%
    \vspace{-1.5em}
    \label{fig:effects}%
\end{figure}

  %
 \vspace{-0.1cm}
\section{Conclusions}
In this paper, we study the tradeoffs for real-time identification of suspicious events when there are  operational capacity restrictions. By separating the learning phase from the operational decision phase, we characterize the minority-class detection rate directly as a function of the inspection resources and the learned classifier predictions. We formulate the streaming arrival of events as a non-homogeneous Poisson process, and analytically derive this tradeoff for static and adaptive threshold-based  decision making strategies. Our experiments on a public fraud detection dataset show that such formulation could be used for practical applications with limited resources for inspecting streaming data, and that using  adaptive thresholds operates very closely to the upper performance limit resulting from batch processing, especially for very small inspection resources. Future work includes studying the end-to-end tradeoff while also considering the learning phase, and extensions to settings where the misclassification cost of minority-class samples are non-identical.

\vspace{0.2cm}
\noindent {{\textbf{Disclaimer}}
This paper was prepared for informational purposes by
the Artificial Intelligence Research group of JPMorgan Chase \& Co. and its affiliates (``JP Morgan''),
and is not a product of the Research Department of JP Morgan.
JP Morgan makes no representation and warranty whatsoever and disclaims all liability,
for the completeness, accuracy or reliability of the information contained herein.
This document is not intended as investment research or investment advice, or a recommendation,
offer or solicitation for the purchase or sale of any security, financial instrument, financial product or service,
or to be used in any way for evaluating the merits of participating in any transaction,
and shall not constitute a solicitation under any jurisdiction or to any person,
if such solicitation under such jurisdiction or to such person would be unlawful.
}


\bibliographystyle{named}
\bibliography{references}


\appendix

\section{Non-Homogeneous Poisson Process}\label{app:nhhp}
This section provides known properties of NHPPs that have been referred to in the main paper, or that have been used in the technical proofs. For more properties and detailed proofs of the properties please see \cite{daley2007introduction,ross2014introduction}.

 \begin{proposition}[Independent splitting (thinning) of a NHPP]\label{prop:splitting}
The independent splitting of a NHPP with intensity  $\lambda(t)$ into $n$ split processes with splitting functions $p_i(t)=p_i, t\geq 0$, produces $n$ independent NHPPs with intensities  $\lambda_i(t) \coloneqq \lambda(t)p_i(t)$.
\end{proposition}

\begin{proposition}[Superposition of independent NHPPs]\label{prop:superposition}
The superposition of $n$ independent NHPPs is itself a NHPP with an intensity equal to the sum of the component intensities.
\end{proposition}

 \begin{proposition}\label{prop:prob q}
Let $N_A(t)$ and $N_B(t)$ be two independent NHPPs with respective intensity functions $\lambda_A(t)$ and $\lambda_B(t)$, and let $N(t)= N_A(t)+N_B(t)$. Then,
\begin{itemize}
\item $N(t)$ is a NHPP with intensity function  $\lambda_A(t)+\lambda_B(t)$.
\item Let
$I_i = 
\begin{cases}
1, &\text{if $i^\text{th}$ event is from process A} \\
0, &\text{otherwise}
\end{cases}
$

Random variables $I_i, i=1,2,\dots$ are $i.i.d.$ Bernoulli($q$), with $q = \lambda_A(t)/(\lambda_A(t)+\lambda_B(t))$.  
\end{itemize}
\end{proposition}

\begin{proposition}\label{prop:conditioning}
Arrival times $t_1,t_2,\dots$ corresponding to a NHPP with intensity $\lambda(t)$, are dependent; however, when conditioned on previous arrivals we have the following. If an event arrives at time $\gamma$, independent of all arrivals prior to $\gamma$, the random wait time to the next event, denoted by $T_\gamma$, is distributed as
\begin{equation}
    f_{T_\gamma}(t) = \lambda(\gamma+t) \exp\Big(- \int_{0}^{t} \lambda(\gamma+u)du\Big)\label{eq:conditional arrival}
\end{equation}
\end{proposition}

\section{Proof of Theorem~\ref{thm:FT}}\label{app:FT}
 As described in Sec.~4, arrival process $N(t)$ can be split into two independent 
 processes $N_0(t)$ and $N_1(t)$. Given that the assigned classifier scores  are  independent from one another, per Proposition~\ref{prop:splitting}, we further split each process $N_i(t)$, $i=0,1$, into two \textit{independent} sub-processes $\widehat{N}_i(t)$ and $\widecheck{N}_i(t)$, corresponding to transactions with scores higher and lower than $\alpha$, respectively, such that
\begin{itemize}[leftmargin=*]
    \item Process $\widehat{N}_0(t)$ with intensity function $\widehat{\lambda}_{0,\alpha}(t) = (1-\beta) (1-F_0(\alpha))\lambda(t)$ corresponds to \textit{non-fraudulent} transactions with scores \textit{exceeding}  $\alpha$.
    \item Process $\widehat{N}_1(t)$ with intensity function $\widehat{\lambda}_{1,\alpha}(t) = \beta (1-F_1(\alpha))\lambda(t)$ corresponds to \textit{fraudulent} transactions with scores \textit{exceeding}  $\alpha$.
\end{itemize}  
Let $I_i$ denote the event that  transaction $i$ arriving at time $t\geq0$ with a score higher than $\alpha$, belongs to process $\widehat{N}_1(t)$.  Given Proposition~\ref{prop:prob q}, random variable $I_i$ is Bernoulli with parameter
 \begin{align}
 q(\alpha) &\coloneqq   \frac{\widehat{\lambda}_{1,\alpha}(t)}{ \widehat{\lambda}_{0,\alpha}(t)+\widehat{\lambda}_{1,\alpha}(t) } 
 = \frac{\beta(1-F_1(\alpha)) }{1-F_S(\alpha)} \label{eq:q FT}
 \end{align} 

 Given a classifier with perfect predictive power, with a properly set threshold, the number of inspections $n_k$ is no more than the overall number of transactions that exceed the threshold in interval $[0,\tau]$, which belong to process $\widehat{N}(\tau)\coloneqq \widehat{N}_0(\tau)+\widehat{N}_1(\tau)$ with expected value  $(1-F_{S}(\alpha))\Lambda(\tau)$. For too low of a  threshold value, transactions with scores higher than $\alpha$ arriving after the first $n_k$ ones will not be inspected due to the capacity restriction. To this end, let $\bar{n}_k \coloneqq \min\{n_k, \widehat{N}(\tau)\}$.  As a result of Proposition~\ref{prop:prob q}, since the  sum of $\bar{n}_k$ $i.i.d.$ Bernoulli  random variables has Binomial distribution, the probability that exactly $j$ transactions among the first $\bar{n}_k$ arrivals exceeding $\alpha$ belong to process $\widehat{N}_1(t)$, and are therefore fraudulent, is $\binom{\bar{n}_k}{j} q(\alpha)^j (1-q(\alpha))^{\bar{n}_k-j}$. Random variable $j$ is $\text{Binomial}(\bar{n}_k, q(\alpha))$.

The total number of fraudulent arrivals in $[0,\tau]$ is ${N}_1(\tau)$, and  therefore, $\Psi_{FT}(k)$, is given by

\begin{align}
     &\Psi_{FT}(k) = 
    \mathbb{E}   \bigg[ \frac{ 1 }{{N}_1(\tau)  } \sum\limits_{j=1}^{\bar{n}_k} j\times 
    \mathbb{P}\Big(\parbox{11em}{ \,among the first $n_k$ arrivals,\\  $j$ are fraudulent} \Big) \bigg] \notag\\
    &  \geq  \frac{1 }{ \mathbb{E}  [{N}_1(\tau)]}\mathbb{E}   \bigg[\sum_{j=1}^{\bar{n}_k}   j\, \binom{\bar{n}_k}{j} q(\alpha)^j (1-q(\alpha))^{\bar{n}_k-j}  \bigg] \label{eq:FT div}\\
     &  =   \frac{1 }{ \beta\Lambda(\tau)}  \mathbb{E}   \Big[  \bar{n}_kq(\alpha)\Big] =  \frac{q(\alpha) }{ \beta\Lambda(\tau)}  \mathbb{E}   \Big[  \bar{n}_k\Big]
     \label{eq:bern mean}\\
     &= \frac{q(\alpha) }{ \beta\Lambda(\tau)} \min\Big\{  {n}_k ,\, (1-F_S(\alpha)) \Lambda(\tau)    \Big\}\\
     &=   ( 1-F_1(\alpha) )  \min\Big\{\frac{ n_k } {\Lambda(\tau)(1-F_S(\alpha))}, \, 1   \Big\}\label{eq:FT last} \\
      &\geq (1-F_1(\alpha)) \min\Big\{\frac{ k-1/\Lambda(\tau)}{ 1-F_S(\alpha) }
      , \,1    \Big\}  
    \end{align}
where \eqref{eq:FT div} follows since $\mathbb{E}[x/y]\geq  \mathbb{E}[x]/\mathbb{E}[y] $ for  $\mathbb{E}[x]\geq0$, and \eqref{eq:bern mean} results from the fact that $\sum_{j=1}^{n}j \binom{n}{j} p^j (1-p)^{n-j} = np$ is the expected value of $j\sim \text{Binomial}(n,p)$. Eq.~\eqref{eq:FT last} results from replacing $q(\alpha)$ from \eqref{eq:q FT}.

\section{Proof of Theorem~\ref{thm:FT-alpha}}\label{app:FT-alpha} 
Given that in the setting of this paper we are interested applications with a large number of expected arrivals in $[0,\tau]$, for simplicity, let us assume that $\Lambda(\tau)\gg 1$, and therefore $1/\Lambda(\tau)\simeq 0$. Then, from  Theorem~2 
it follows that for a given threshold $\alpha$, the detection rate can be rewritten as follows 
\begin{equation}\label{eq:FT-spread}
    \Psi_{FT}(k) = 
    \begin{cases}
        1-F_1(\alpha), & \text{if  } k \geq 1-F_S(\alpha)   \\
      k\frac{(1-F_1(\alpha))}{1-F_S(\alpha)}, & \text{if  }k < 1-F_S(\alpha)  
    \end{cases} 
\end{equation}
When $k\geq 1-F_S(\alpha)$, then $\Psi_{FT}(k)=1-F_1(\alpha)$ is decreasing  in $\alpha$ since the CDF $F_1(s)$ is increasing in $s$. Therefore, it is maximized for the smallest value $\alpha$ for which the condition $k\geq 1-F_S(\alpha)$ is satisfied, i.e., $\alpha^* = {F}_S^{-1}(1 - k)$. When $k< 1-F_S(\alpha)$, then 
$$\Psi_{FT}(k)= \frac{k}{1-F_S(\alpha)}(1-F_1(\alpha)) <1-F_1(\alpha) ,$$
which is less than the value achieved when $k\geq 1-F_S(\alpha)$. Therefore, the maximum is achieved with $\alpha^*$.

Note that eq.~\eqref{eq:FT-spread} has an intuitive interpretation. Suppose $k \geq 1- F_S(\alpha)$, then the capacity exceeds the expected number of transactions with $G(X) \geq \alpha$, of which $(1-F_1(\alpha))$ are fraudulent on average. Otherwise, when $k < 1- F_S(\alpha)$, only a fraction $k/(1 - F_S(\alpha))$ of the transactions with $G(X) \geq \alpha$ are captured on average, of which $1-F_1(\alpha)$ are fraudulent.

 \section{Proof of Theorem~5}\label{app:DT}
 Per Proposition~\ref{prop:splitting}, we denote the arrival process of non-fraudulent and fraudulent transactions with scores exceeding critical curve $\alpha_j(t)$ by $\widehat{N}_{0,\alpha_j}(t)$ and $\widehat{N}_{1,\alpha_j}(t)$, respectively, with intensities  $\lambda_{0,\alpha_j}(t)=(1-\beta)(1-F_0(\alpha_j(t)))\lambda(t)$ and $\lambda_{1,\alpha_j}(t)=\beta(1-F_1(\alpha_j(t))\lambda(t)$. Based on Proposition~\ref{prop:prob q}, a transaction arriving at $t$ with score $S>\alpha_j(t)$ is fraudulent, i.e., belongs to  $\widehat{N}_{1,\alpha_j}(t)$, with probability 
 \begin{align}
 \bar{q}_{\alpha_j}(t)& 
 = \beta \frac{ 1-F_1(\alpha_j(t))}{1-F_S(\alpha_j(t)))}.
 \end{align}

Let $\bar{T}_{\gamma,j}$, $j=1,\dots,n_k$, be the (random) waiting  time starting from $t\geq \gamma$ until the first transaction is selected for inspection when we have $j$ inspections left, i.e., the additional time after $\gamma$ till a transaction score exceeds critical curve $\alpha_{j}(t)$. If the ${(j-1)}^{\text{th}}$, $j=2,\dots,n_k$, inspection happened at time $\gamma$ with respect to critical curve $\alpha_{n_k-j+2}(t)$, then, the $j^\text{th}$ inspection happens at $\gamma+\bar{T}_{\gamma,n_k-j+1}$ with respect to critical curve $\alpha_{n_k-j+1}(t)$. Let us define the index $\rho(j)\coloneqq n_k-j+1$. Based on Proposition~\ref{prop:conditioning}, $\bar{T}_{\gamma,j}$ is distributed according to \eqref{eq:conditional arrival} with intensity function $\lambda_{\alpha_{j}}(t) = (1-F_S(\alpha_{\rho(j)}))\lambda(t)$. 

Let us denote the waiting times between inspections by  ${T}_{1},{T_2},\dots,{T}_{n_k}$. Then, the expected fraction of fraudulent transactions that are selected for inspected is
\begin{align}
 &\Psi_{DT}(k) 
=   \mathbb{E} \bigg[ \frac{1 }{{N}_1(\tau)  } \sum\limits_{j=1}^{n_k}\mathbbm{1}\Big\{j^\text{th} \text{ inspection belongs to }N_{1,\alpha_{\rho(j)}}(t) \Big\}     \bigg] \notag\\
& \stackrel{(a)}{\geq} \frac{1 }{ \mathbb{E}[{N}_1(\tau)  ] } \sum\limits_{j=1}^{n_k}\mathbb{E} \bigg[ \mathbb{P}\Big(\text{inspection at $\sum\limits_{i=1}^{j} {T}_{i}$ belongs to }N_{1,\alpha_{\rho(j)}}(t) \Big)     \bigg] \notag\\
& = \frac{1 }{ \beta\Lambda(\tau)} \sum\limits_{j=1}^{n_k} \mathbb{E}_{{T}_{1}} ... \mathbb{E}_{{T}_{j}}  \bigg[  \bar{q}_{\alpha_{\rho(j)}}(t)\Big(\sum\limits_{i=1}^{j} {T}_{i} \Big)   \bigg] ,\notag
\end{align}
 where (a) follows since $\mathbb{E}[x/y]\geq  \mathbb{E}[x]/\mathbb{E}[y] $ for  $\mathbb{E}[x]\geq0$.

\section{Proof of Theorem~\ref{thm:RS}}\label{app:RS}
For a given inspection capacity $k$, any of the (non-fraudulent or fraudulent) transactions are equally likely to be selected for inspection with probability  $\frac{n_k}{N(\tau)}$. Then, the expected number of fraudulent transactions arriving according to process $N_1(t)$ that are selected is
 \begin{align}
&\Psi_{RS}(k) = 
 \mathbb{E}_{N,N_1}  \bigg[  \frac{ \mathbb{E}_{S^{(1)}} \Big[ \sum_{i=1}^{N_1(\tau)}   \mathbbm{1}\{i\in \mathcal{I}\}  \Big]}{ N_1(\tau)} \bigg] \notag\\
 &\;\; \stackrel{(a)}{=} \mathbb{E}_{N,N_1}  \bigg[  \frac{ N_1(\tau)\frac{n_k }{N(\tau)}}{ N_1(\tau)}\bigg]  
 = \mathbb{E}_{N} \Big[ \frac{n_k }{N(\tau)} \Big]\stackrel{(b)}{\geq}  \frac{  n_k }{\mathbb{E}_{N} [ {N(\tau)}]} \notag\\
 & \;\; = \frac{  n_k }{\Lambda(\tau)} 
 > \frac{( k \Lambda(\tau) -1)}{\Lambda(\tau)}= k -\frac{1}{\Lambda(\tau)}, \notag
\end{align}
where (a) 
follows since the  sampling  is independent from the transaction being fraudulent, and (b) 
follows from Jensen's inequality since the function $g(x)=\frac{1}{x}$ is convex for $x>0$, and therefore $\mathbb{E}[1/x]\geq1/\mathbb{E}[x]$.

\section{Proof of Theorem~\ref{thm:BP}}\label{app:BP}
We derive the  detection rate  using the following definition.  
\begin{definition}\label{def:order stats}
${X}_1,\dots,{X}_n$ are $n$ $i.i.d.$ continuous random variables with a common PDF $f_X(x)$. The ordered realizations of the random variables, ${X}_{(1)},\dots,{X}_{(n)}$, sorted in increasing order, are also random variables. ${X}_{(i)}$ is referred to as the $i^\text{th}$ order statistic.
\end{definition}

 As per definition~\ref{def:order stats}, we denote the sorted scores assigned to non-fraudulent and fraudulent transactions in increasing order, respectively, by ${S}^{(0)}_{(1)}\leq ...\leq{S}^{(0)}_{({N_0(\tau)})}$, and ${S}^{(1)}_{(1)}\leq ...\leq{S}^{(1)}_{({N_1(\tau)})}$. 

 Among the fraudulent transactions, consider  the one with the $j^\text{th}$ largest score, which is equivalent to the  $(N_1(\tau)-j+1)^\text{th}$ order statistic of a sample of size $N_1(\tau)$ with PDF $f_1(S)$. Let us define the index $\rho_1(j)\coloneqq N_1(\tau)-j+1$. Similarly, among the non-fraudulent transactions, consider  the one with the $(n_k-j+1)^\text{th}$ largest score, which is equivalent to the  $(N_0(\tau)-n_k+j)^\text{th}$ order statistic of a sample of size $N_0(\tau)$ with PDF $f_0(S)$. Let us define the index $\rho_0(j)\coloneqq N_0(\tau)-n_k+j$. Then, for any $j\in\{1,\dots,\min\{n_k,N_1(\tau)\}$, the fraudulent transaction with the $j^\text{th}$ largest score is selected for inspection, only if
\begin{align}
 S^{(1)}_{ (\rho_1(j))}> S^{(0)}_{(\rho_0(j))}.
\end{align}
Therefore, the fraction of fraudulent transactions that are inspected, is given by
 \begin{align}
 & \Psi_{BP}(k) =   \mathbb{E}  \Big[\frac{1}{N_1{(\tau)}}   \sum\limits_{j=1}^{\min\{n_k,N_1(\tau)\}} \mathbbm{1} \Big\{ S^{(1)}_{ (\rho_1(j))}> S^{(0)}_{(\rho_0(j))}\Big\}  \Big]      \notag\\
  & \geq  \frac{1 }{ \mathbb{E}[{N}_1(\tau)  ] }  \mathbb{E}_{N_0,N_1}  \bigg[      \sum\limits_{j=1}^{\min\{n_k,N_1(\tau)\}} \mathbb{P} \Big( S^{(1)}_{ (\rho_1(j))}> S^{(0)}_{(\rho_0(j))}\Big)        \bigg] \label{eq:BP div}\\
  & = \frac{1 }{ \beta\Lambda(\tau)}  \mathbb{E}_{ N_0,N_1}  \Big[     \sum\limits_{j=1}^{\min\{n_k,N_1(\tau)\}} F_{W_{j}}  (0)\Big],
\end{align}
where \eqref{eq:BP div} follows since $\mathbb{E}[x/y]\geq  \mathbb{E}[x]/\mathbb{E}[y] $ for  $\mathbb{E}[x]\geq0$, and $ W_{j} \coloneqq   S^{(0)}_{\rho_0(j)} - S^{(1)}_{\rho_1(j)}$, is a random variable with CDF 
\begin{align}
F_{ W_{j}}(w) 
   &= \int_{0}^1 \int_{0}^{w+s_1} f_{S^{(0)}_{\rho_0(j)}}(s_0) f_{S^{(1)}_{\rho_1(j)}}( s_1) ds_0 ds_1 \label{eq:indep}
\end{align}
and  the PDF of the order statistic $f_{{S}_{(i)}}(s)$ is given by Lemma \eqref{lemma:orderstat}, provided in the following.

\begin{lemma}\cite[Theorem 5.4.4]{casella2002statistical}\label{lemma:orderstat}
Let $X_{(i)}$ be the $i^{th}$ order statistics of an $i.i.d.$ random sample of size $n$ from a continuous distribution with CDF $F_X(x)$ and PDF $f_X(x)$. Then, the PDF of $X_{(i)}$ is  
\begin{align}
    f_{X_{(i)}}(x) = i\binom{n}{i} \Big[ F_X(x)\Big]^{i-1} \Big[ 1- F_X(x)\Big]^{n-i} f_{X}(x). \label{eq:orderstat}
\end{align}
\end{lemma}

 \section{End-to-End Upper Bound}\label{app:UB}
In this paper, we consider approaches that identify suspicious events in two steps: (1) learning a predictive classifier that  discovers frauds based on the transaction features $X$, and (2) threshold-based sampling based on the classifier scores. We have mainly focused on the second step, and compared sampling methods for a given classifier. By considering the end-to-end system that takes both steps into account simultaneously,  
we can define the fraud detection rate as a function of the inspection capacity, denoted by $\mathbf{\Psi}(k)$, as follows:
 \begin{align}\label{eq:opt psi}
   \mathbf{\Psi}(k) \coloneqq \max_{G\in \mathcal{G}} \max_{\mathcal{I}(k)} \Psi(k)
 \end{align}
The following theorem provides an upper bound on $\mathbf{\Psi}(k)$, which is also an upper bound on ${\Psi}(k)$ defined in eq. (1) of the main paper. Note that this bound is not necessarily achievable.

\begin{theorem}[\bf{End-to-End Upper Bound}]\label{thm:up}
For a given inspection capacity $k$, the end-to-end tradeoff satisfies
\begin{align}
\mathbf{\Psi}(k)  \leq  \min\Big\{ k/\beta,\,1\Big\}.
\end{align}
\end{theorem}
\begin{proof}
The ideal classifier would be able to, given an inspection budget of $n_k$ transaction, perfectly select (the first) $n_k$ fraudulent transactions as the arrive. With enough capacity, i.e., when  $n_k \geq N_1(\tau)$, the inspector will  detect all frauds without error. Therefore, 
\begin{align}
 &\mathbf{\Psi}(k)  = \mathbb{E}  \Big[ \min\Big\{ \frac{n_k}{N_1(\tau)},\, 1 \Big\}\Big]    \geq      \min\Big\{ \frac{n_k}{\mathbb{E}[N_1(\tau)]},\, 1 \Big\}\notag \\
 &=    \min\Big\{  \frac{  n_k }{\beta \Lambda(\tau)}, 1\Big\} > \min\Big\{  \frac{( k \Lambda(\tau) -1)}{\beta\Lambda(\tau)}, 1\Big\} 
\geq \min\Big\{ \frac{k}{\beta},\,1\Big\}  \notag
\end{align}
which follows   from Jensen's inequality since  function $g(x)=\frac{1}{x}$ is convex for $x>0$, and therefore $\mathbb{E}[1/x]\geq1/\mathbb{E}[x]$.

\end{proof}

 \section{Experiment Setup Details}\label{app:experiments}
 
\paragraph{Data Prepossessing:}
The IEEE-fraud dataset contains more than one million online  transactions, and each transaction contains more than $400$ features. While the original dataset contains a train set and test set, only the train set contains fraud labels, and therefore, we only use the train set for our analysis which has $590$ K  samples. Each transaction has a timestamp feature (in seconds) that is provided with respect to a reference value. We assume that each time interval $[0,\tau]$ corresponds to a full day of $\tau=86400$ seconds, resulting in $183$ episodes for our experiments.  We follow the analysis provided in \cite{ieee-processing} for preprocessing and removing the redundant features through correlation analysis, and we further engineer aggregate features that increase the local validation AUC.

As described in Sec.~5 of the paper, the original dataset contains around $3.5\%$ frauds. We create two additional datasets by modifying the number of fraudulent transactions by up-sampling (\cite{chawla2002smote}) and down-sampling (uniformly at random) the minority class. The average number of daily transactions $\Lambda(\tau)$, and average number of daily frauds $\Lambda_1(\tau)$ are provided in Table.~\ref{tab:data}

 \begin{table}\setlength{\tabcolsep}{3pt}
\centering
\resizebox{ \linewidth}{!}{%
\begin{tabular}{lccc}
\toprule
 &  Class   & Average Number of   &  Average Number of        \\
  &  Imbalance   &            Daily Transactions  &  Daily Frauds       \\
  &  $\beta$ &  $\Lambda(\tau)$  &   $\Lambda_1(\tau)$    \\
 Dataset &    &     &       \\
\midrule
IEEE-fraud-$8\%$ & $0.08$ & $3378 \pm 858$   &  $271 \pm 68$  \\
\midrule
 IEEE-fraud-$3.5\%$  &  $0.035$  & $3219 \pm 846$ &$112 \pm 28$  \\
\midrule
  IEEE-fraud-$1.5\%$ & $0.015$ & $3154 \pm 843$ & $48 \pm 13 $  \\
 \bottomrule
\end{tabular}
}
\caption{Properties of IEEE-fraud dataset.}
\label{tab:data}
\end{table}

\vspace{0.1cm}
\noindent{\textbf{Training Setup}}
We split the dataset into three portions of by randomly selecting the days, such that there are $93$ days for training the classifier, $45$ days for estimating $\lambda(t)$, $F_0(t)$ and $F_1(t)$, and there are $45$ days for empirical experiments. The parameters for training each classifier based on the area under the ROC curve (AUC) are as followed:
\begin{itemize}
\item \textbf{Gradient Boosted Decision Trees (GBDT):} We use the \textbf{XGBoost} library to train our classifier, using $25\%$ of the data for cross-validation, with the following parameters.
\begin{lstlisting}[language=python]
{   'n_estimators': 2000, 
    'max_depth': 12,
    'learning_rate' : 0.02,
    'subsample' : 0.8,
    'colsample_bytree': 0.4,
    'missing': -1,   
    'eval_metric': 'auc }
\end{lstlisting}

\item \textbf{Random Forests (RF):}
We use the \textbf{scikit-learn} library to train our classifier, by using grid search cross-validation over the following hyperparameters in $50$ iterations, and use the set of parameters for our final training.
\begin{lstlisting}[language=python]
{   'bootstrap': True, 
    'max_depth': [5, 12],
    'max_features' : [2, 3],
    'min_samples_leaf': [3, 4, 5],
    min_samples_split': [8, 10, 12],
    'n_estimators': [100, 200, 300, 1000]}
\end{lstlisting}

\item \textbf{Logistic regression (LR):}
We use the \textbf{scikit-learn} library to train our classifier, using by using grid search cross-validation in $50$ iterations over 
\begin{lstlisting}[language=python]
{  'C': np.power(10.0, np.arange(-3, 3)) }
\end{lstlisting}

\end{itemize}

 \begin{figure*} [!ht]
    \captionsetup[subfigure]{aboveskip=-1pt,belowskip=-2em}
    \centering
    \subfloat[\centering ]{{\includegraphics[width=0.23\textwidth]{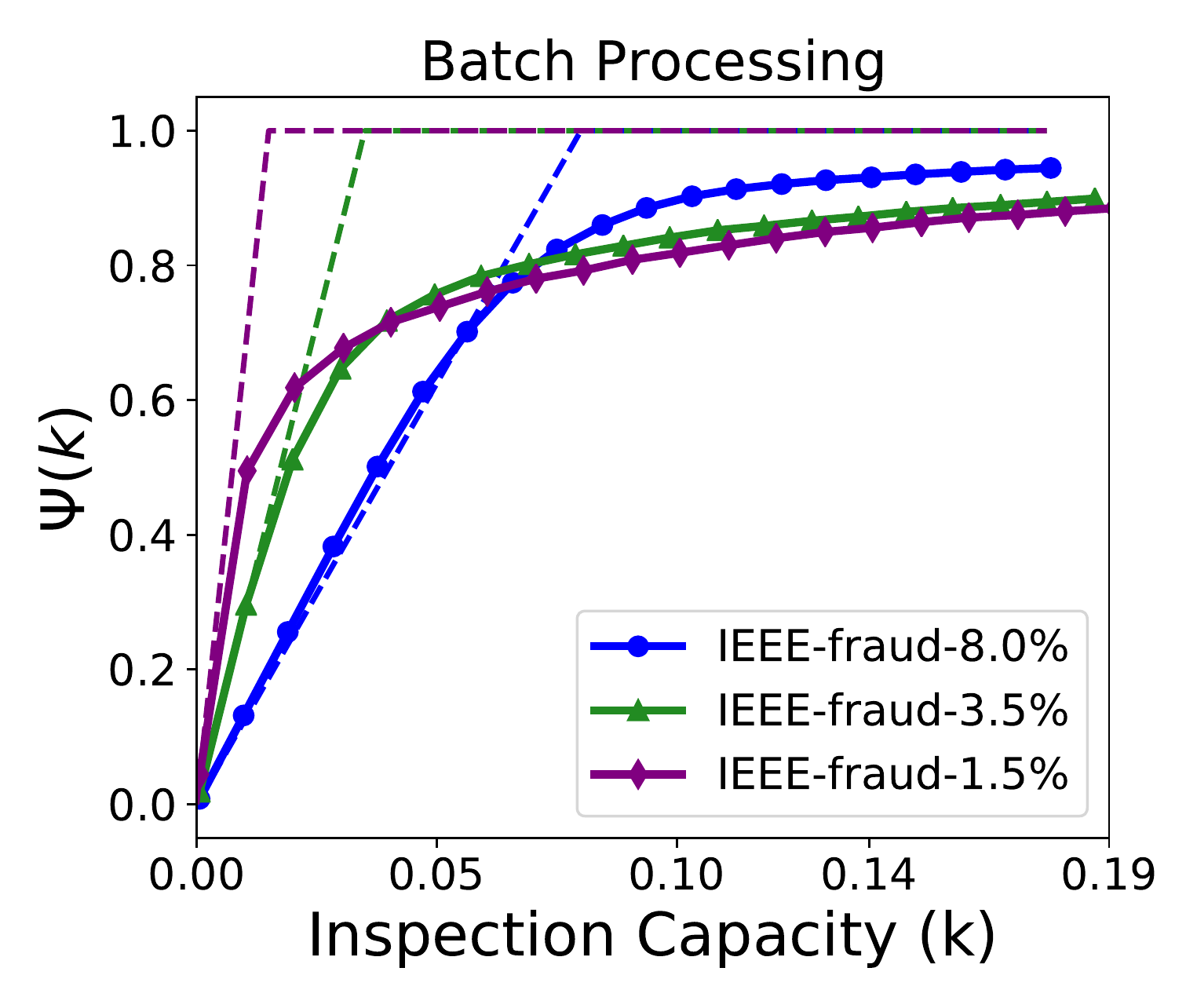} }}%
    \subfloat[\centering ]{{\includegraphics[width=0.23\textwidth]{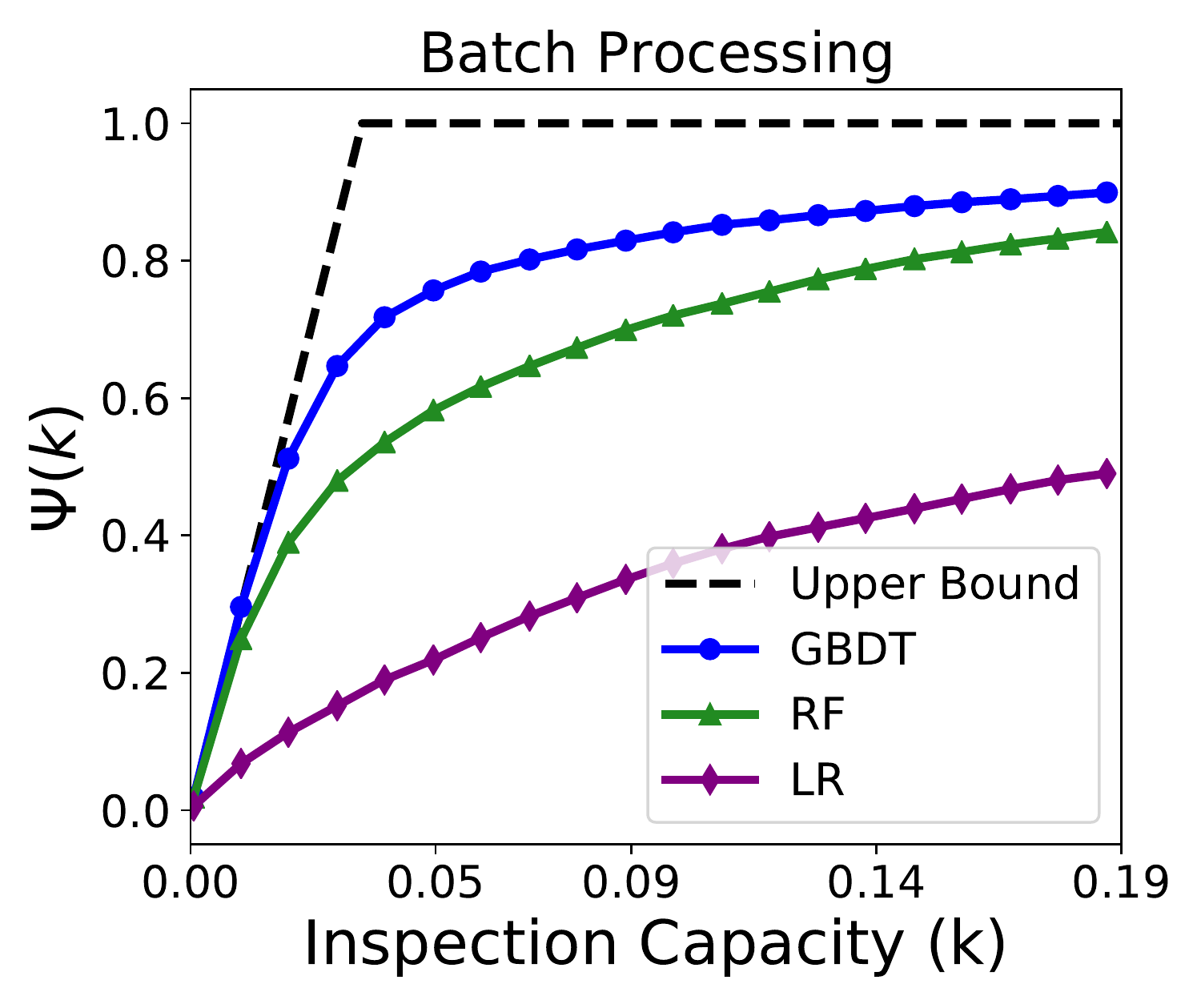} }}%
    \subfloat[\centering ]{{\includegraphics[width=0.23\textwidth]{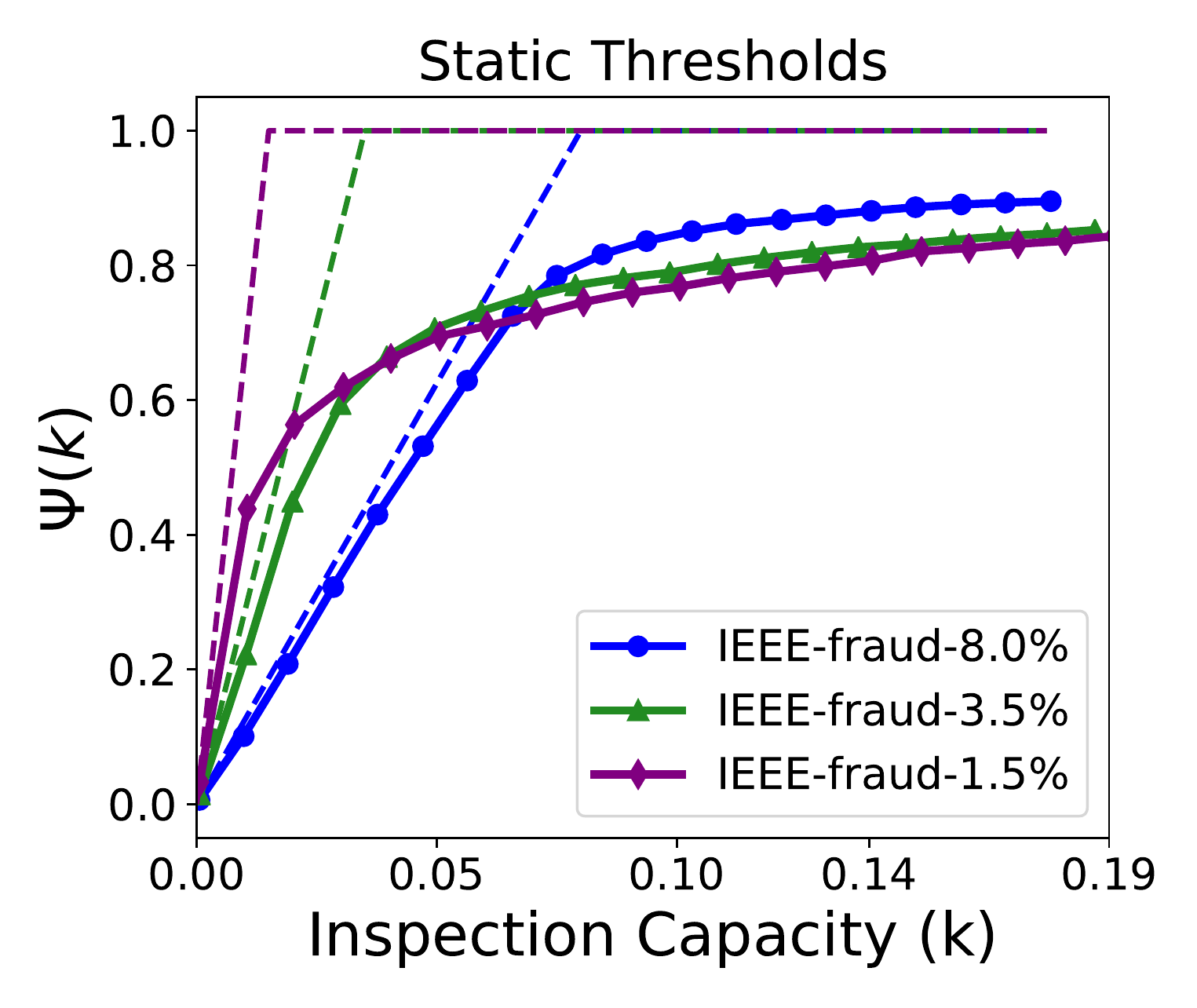} }}%
    \subfloat[\centering ]{{\includegraphics[width=0.23\textwidth]{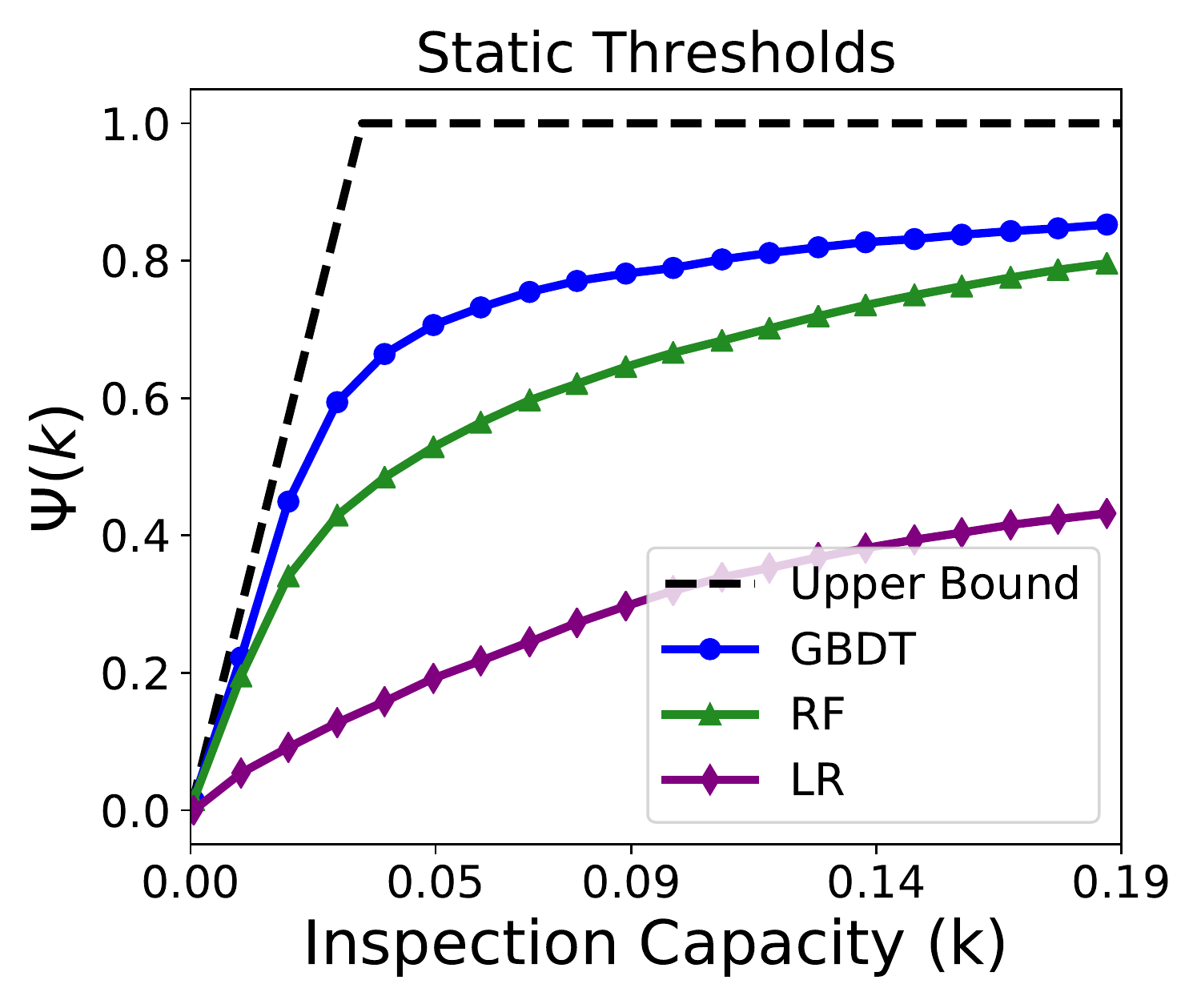} }}%
     \caption{Tradeoff with dynamic thresholds for batch processing (a, b), and static thresholds (c, d). 
    }%
    \label{fig:additional}%
\end{figure*}
\vspace{0.1cm}
\noindent{\textbf{Additional Experiment Results}}
In this section, we provide results similar to Fig.~4 corresponding to the Batch Processing (Fig.~\ref{fig:additional} (a), (b)) and Static Thresholds methods (Fig.~\ref{fig:additional} (c), (d)). Both methods exhibit similar trends as the method with dynamic thresholds. Specifically, for the impact of class imbalance on the left, it is observed that For very small capacities, the results are very close to the upper bound, and as the capacity increases, class imbalance impacts the detection rate more severely.  And, the classifier learning phase impacts the tradeoff such that a classifier with inferior predictive power (low AUC) selects non-fraudulent transactions early-on, and operates further from the upper bound.

\end{document}